\documentclass[10pt]{article}

\usepackage{natbib}
\usepackage[T1]{fontenc}
\usepackage{lmodern}
\usepackage{fancyhdr}
\setcitestyle{authoryear,round,citesep={;},aysep={,},yysep={;}}

\makeatletter
\def\addcontentsline#1#2#3{}

\newlength\aftertitskip
\newlength\beforetitskip
\newlength\interauthorskip
\newlength\aftermaketitskip
\def\@startauthor{\noindent \normalsize\bf}
\def\@endauthor{}
\def\addr{\small\it}
\def\email{\hfill\small\it}
\def\name{\normalsize\bf}
\def\AND{\@endauthor\rm\hss \vskip \interauthorskip \@startauthor}

\def\maketitle{\par
\begingroup
   \def\thefootnote{\fnsymbol{footnote}}
   \def\@makefnmark{\hbox to 0pt{$^{\@thefnmark}$\hss}}
   \long\def\@makefntext##1{\parindent 1em\noindent
                            \hbox to1.8em{\hss $\m@th ^{\@thefnmark}$}##1}
   \@maketitle \@thanks
\endgroup
\setcounter{footnote}{0}
\let\maketitle\relax \let\@maketitle\relax
\gdef\@thanks{}\gdef\@author{}\gdef\@title{}\let\thanks\relax}

\def\@maketitle{\vbox{\hsize\textwidth
{\LARGE\bf\sffamily \@title\par}\vskip \aftertitskip
\@startauthor \@author \@endauthor
\vskip 0.3in minus 0.1in}}

\renewenvironment{abstract}{\vskip.075in\centerline{\large\bf\sffamily
Abstract}\vspace{0.5ex}\begin{quote}}{\par\end{quote}\vskip 1ex}

\def\section{\@startsection {section}{1}{\z@}{-2.0ex plus
    -0.5ex minus -.2ex}{1.5ex plus 0.3ex
minus0.2ex}{\large\bf\raggedright\sffamily}}
\def\subsection{\@startsection{subsection}{2}{\z@}{-1.8ex plus
-0.5ex minus -.2ex}{0.8ex plus .2ex}{\normalsize\bf\raggedright\sffamily}}
\def\subsubsection{\@startsection{subsubsection}{3}{\z@}{-1.5ex
plus      -0.5ex minus -.2ex}{0.5ex plus
.2ex}{\normalsize\bf\raggedright\sffamily}}
\def\paragraph{\@startsection{paragraph}{4}{\z@}{1.5ex plus
0.5ex minus .2ex}{-1em}{\normalsize\bf}}
\def\subparagraph{\@startsection{subparagraph}{5}{\z@}{1.5ex plus
  0.5ex minus .2ex}{-1em}{\normalsize\bf}}

\skip\footins 9pt plus 4pt minus 2pt
\def\footnoterule{\kern-3pt \hrule width 12pc \kern 2.6pt}
\leftmargini\leftmargin
\def\@listi{\leftmargin\leftmargini}
\def\@listii{\leftmargin\leftmarginii
   \labelwidth\leftmarginii\advance\labelwidth-\labelsep
   \topsep 2pt plus 1pt minus 0.5pt
   \parsep 1pt plus 0.5pt minus 0.5pt
   \itemsep \parsep}
\def\@listiii{\leftmargin\leftmarginiii
    \labelwidth\leftmarginiii\advance\labelwidth-\labelsep
    \topsep 1pt plus 0.5pt minus 0.5pt
    \parsep \z@ \partopsep 0.5pt plus 0pt minus 0.5pt
    \itemsep \topsep}
\def\@listiv{\leftmargin\leftmarginiv
     \labelwidth\leftmarginiv\advance\labelwidth-\labelsep}
\def\@listv{\leftmargin\leftmarginv
     \labelwidth\leftmarginv\advance\labelwidth-\labelsep}
\def\@listvi{\leftmargin\leftmarginvi
     \labelwidth\leftmarginvi\advance\labelwidth-\labelsep}
\belowdisplayskip \abovedisplayskip
\makeatother

\usepackage{amsmath,amsfonts,bm}

\def\eqref#1{equation~\ref{#1}}

\def\1{\bm{1}}

\DeclareMathAlphabet{\mathsfit}{\encodingdefault}{\sfdefault}{m}{sl}
\SetMathAlphabet{\mathsfit}{bold}{\encodingdefault}{\sfdefault}{bx}{n}

\usepackage{hyperref}
\usepackage{url}
\usepackage{svg}
\usepackage[ruled,algo2e]{algorithm2e}
\usepackage{graphicx}
\usepackage{amsmath}
\usepackage{amssymb}
\usepackage{mathtools}
\usepackage{amsthm}
\usepackage{tabularray}
\usepackage{subcaption}
\usepackage{ulem}
\usepackage{multirow}
\usepackage{multicol}
\usepackage{placeins}
\newtheorem{assumption}{Assumption}
\newtheorem{theorem}{Theorem}
\newtheorem{lemma}{Lemma}
\newtheorem{corollary}{Corollary}
\newtheorem{remark}{Remark}
\usepackage{tabularx}
\usepackage{booktabs}
\usepackage{float}

\title{Decentralized SGD with Controlled Disagreement\\Finds Flatter Minima}

\author{\name Zesen Wang \email zesen@kth.se \\
      \addr Division of Decision and Control Systems\\
      KTH Royal Institute of Technology
      \AND
      \name Mikael Johansson \email mikaelj@kth.se \\
      \addr Division of Decision and Control Systems\\
      KTH Royal Institute of Technology}

\begin{document}

\maketitle

\begin{abstract}
Decentralized training is often regarded as inferior to centralized training because the consensus errors between workers are thought to undermine convergence and generalization.
This work challenges this view by introducing decentralized SGD with Adaptive Consensus (\textsc{DSGD-AC}), which uses a time-dependent scaling mechanism to maintain consensus errors throughout the training.
We show that adaptive consensus changes the stationary variance of disagreement modes by balancing two effects: it preserves consensus-error magnitude through weaker graph damping while still allowing curvature-dependent damping to shape the disagreement directions.
This balance can produce a stronger Hessian-weighted loss-envelope penalty around the deployed model, even when normalized Hessian alignment is weaker than in standard DSGD.
Empirical results on image classification show that \textsc{DSGD-AC} reaches flatter solutions and higher test accuracy than standard DSGD and even centralized SGD.
Together, these results support consensus errors as a useful implicit regularizer and open a new perspective on the design of decentralized learning algorithms.
\end{abstract}

\section{Introduction}

In large-scale deep learning, decentralized training lets workers exchange model parameters only with neighbors, avoiding the cost of global synchronization and all-reduce communication \citep{abadi2016tensorflow, li2020pytorch}. This reduces latency and eliminates single points of failure, making decentralized approaches attractive for geographically distributed systems \citep{dhasade2023decentralized, gholami2024digest} and GPU clusters \citep{lian2017can, assran2019stochastic, wang2025promise}.

Despite its practical appeal, decentralized training methods such as the decentralized stochastic gradient descent (DSGD) are widely regarded as inferior to centralized training in terms of convergence and generalization, even when data distributions of workers are i.i.d. This gap was largely attributed to consensus errors, the persistent discrepancies in the model parameters between different workers \citep{alghunaim2022unified,zhu2022topology}. Prior work has therefore focused on reducing these errors through improved communication topologies \citep{ying2021exponential,takezawa2023beyond} and algorithm designs \citep{pu2021distributed, wang2019slowmo,lin2021quasi} to make decentralized training closer to its centralized counterpart.

However, this perspective neglects the potential constructive role of consensus errors. Rather than acting as detrimental noise, these discrepancies may serve as structured perturbations that encourage preference for flatter minima in the loss landscape, solutions that strongly correlate with better generalization \citep{jiang2019fantastic}. This view draws inspiration from sharpness-aware minimization strategies \citep{foret2020sharpness, bisla2022low, li2024friendly}, which improve model robustness and generalization by explicitly introducing curvature-aware perturbations.

In this study, we challenge the conventional view by introducing Decentralized SGD with Adaptive Consensus (DSGD-AC), an algorithm that strategically maintains consensus errors through an adaptive, time-dependent scaling mechanism.
Our theory characterizes the stationary variance of disagreement modes and shows how the adaptive consensus factor balances the magnitude and curvature profile of consensus errors, inducing a stronger Hessian-weighted penalty around the deployed global average. We verify the theory with comprehensive empirical results and diagnostics on Hessian information. DSGD-AC incurs negligible additional computational cost relative to standard SGD or DSGD, while retaining the communication efficiency and runtime advantages of decentralized training.

The main contributions of this work are threefold. First, we propose DSGD-AC, an adaptive consensus algorithm that maintains theoretically motivated consensus errors, finds flatter minima, and improves generalization on deep learning tasks at minimal computational cost. Second, we provide theoretical analysis and empirical diagnostics of the modal variance and consensus error radius, showing that the controlled errors induce a stronger Hessian-weighted penalty and are associated with flatter solutions. Third, we evaluate DSGD-AC against strong baselines on deep learning benchmarks, verifying its practical effectiveness.

\subsection{Related work}

\paragraph{Canonical view of consensus errors} The prevailing view in decentralized training is that models should minimize the consensus errors and approximate centralized training as closely as possible. To mitigate discrepancies among local models caused by weakly connected networks, prior work has focused on tracking global information \citep{wang2019slowmo, pu2021distributed, yuan2021decentlam, takezawa2022momentum}, enhancing communication topologies to improve convergence rates \citep{ying2021exponential, zhu2022topology, takezawa2023beyond}, and more. 
In addition, several theoretical studies \citep{zhu2022topology, alghunaim2022unified} establish a theoretical connection between the connectivity of decentralized communication topologies and both convergence and generalization, demonstrating that weaker connectivity results in poorer outcomes on both fronts.
In contrast, we demonstrate the constructive role of consensus errors by showing that controlled disagreement can induce a meaningful Hessian-weighted loss-envelope penalty, and we propose DSGD-AC as a practical way to exploit this effect in deep learning tasks.

\paragraph{Explorations beyond the canonical view} Since the canonical perspective dominates, research into the potential benefits of consensus errors remains sparse. \cite{kong2021consensus} empirically identified error thresholds and noted advantages in specific training phases; however, they did not explore regimes where errors exceed those of ring-topology DSGD, and their control scheme showed limited gains. \cite{zhu2023decentralized} offers a novel interpretation, framing consensus errors in DSGD as random perturbations around the global average, which are asymptotically equivalent to average-direction SAM \citep{bisla2022low}. Our work further identifies how the magnitude and Hessian-weighted effect of consensus errors are shaped by gradient noise and curvature-dependent damping, and, by proposing DSGD-AC, shows how decentralized training can use disagreement constructively beyond the large-batch setting.

\paragraph{Explicit curvature-related perturbations} If the global average is taken as the deployed model \citep{zhu2023decentralized}, decentralized learning can be interpreted as training the average model, and worker deviations act as perturbations around it. Sharpness-aware minimization (SAM) was first proposed by \cite{foret2020sharpness} to improve the generalization of deep neural networks, and many variants \citep{kwon2021asam, bisla2022low, liu2022random, li2024revisiting, luo2024explicit} were developed to improve it further. However, to achieve the best performance, these algorithms typically require additional gradient evaluations and increase the computational cost significantly. Our work uses consensus errors as free perturbations that can enhance generalization without introducing extra computation or memory than DSGD.

\section{Problem setting and notation}

We consider decentralized training in a data center where the full dataset is accessible to all workers. The training uses a standard distributed data sampler: at the beginning of each epoch, the dataset is reshuffled and evenly partitioned across workers, yielding i.i.d. data distributions in expectation. This setting corresponds to the standard decentralized optimization regime studied by \cite{assran2019stochastic, ying2021exponential, kong2021consensus, zhu2023decentralized, wang2025promise}.

\paragraph{Decentralized optimization}
We consider 
$n$ workers collaborating to train a $d$-dimensional model. Let $[n]$ denote the set of integers $\{1, 2, \dots, n\}$. 
Each worker $i\in [n]$ holds a local objective determined by its local dataset $\mathcal{D}_i$:
\begin{equation}
    f_i(x)=\mathbb{E}_{s\sim \mathcal{D}_i}[f_i(x;s)].
\end{equation}
Under the i.i.d. setting considered here, all workers share the same underlying objective but operate on different mini-batches drawn from the shared dataset. Each worker maintains its own local model $x_i$. The workers collaboratively solve
\begin{equation}\label{eq:do-obj}
    \begin{array}[c]{rl}
    \underset{\{x_1,x_2,\cdots,x_n\}}{\text{minimize}}&  F(x_1,\cdots,x_n)=\frac{1}{n}\sum_{i=1}^n f_i(x_i),\\
    \text{subject to } & x_i=x_j,\qquad \forall i,j\in[n],
    \end{array}
\end{equation}
where the constraint enforces consensus across workers.

\paragraph{Decentralized SGD (DSGD)} The update of DSGD \citep{lian2017can} on worker $i$ is:
\begin{equation}\label{eq:dsgd-update}
    x_i^{(t)}=x_i^{(t-1)}-\alpha^{(t)}\nabla f(x_i^{(t-1)};s_i^{(t)})+\sum_{j\in\mathcal{N}(i)} W_{ij}(x_j^{(t-1)}-x_i^{(t-1)})
\end{equation}
where $\mathcal{N}(i)$ is the neighbor set of worker $i$ (including itself), $W$ is a symmetric, non-negative, doubly stochastic matrix defining 
with $W_{ij}=0$ if $i\not\in \mathcal{N}_j)$, and $\smash{s_i^{(t)}}$ is the  mini-batch sampled by worker $i$ at iteration $t$.

Following standard notation, we denote the global average by $\bar{x}^{(t)}:=\frac{1}{n}\sum_{i=1}^n x_i^{(t)}$, the consensus errors $\delta_i^{(t)}:=x_i^{(t)}-\bar{x}^{(t)}$ and the matrix forms $X^{(t)}:=[x_1^{(t)},\cdots,x_n^{(t)}]^\top$,  $G^{(t)}:=[g_1^{(t)},g_2^{(t)},\cdots,g_n^{(t)}]^\top$,  $\Delta^{(t)}=[\delta_1^{(t)},\cdots,\delta_n^{(t)}]^\top$, and 
$\bar{X}^{(t)}=[\bar{x}^{(t)},\cdots, \bar{x}^{(t)}]^\top$. We also denote the center projector $P=I-\frac{1}{n}11^\top$ and the graph Laplacian $L=I-W$. Then, the DSGD update in Eq. (\ref{eq:dsgd-update}) can be written in matrix form
\begin{equation}\label{eq:dsgd-update-matrix}
    X^{(t)}=X^{(t-1)}-\alpha^{(t)}G^{(t)}+(W-I)X^{(t-1)}
\end{equation}
We focus on the adapt-while-communication variant, which has been shown to be more runtime efficient \citep{lian2017can,assran2019stochastic,wang2025promise} and to share the same generalization bound as the adapt-then-communication variant \citep{bellet2023improved}.

\subsection{Assumptions}

Our theoretical analysis relies on the following assumptions.

\begin{assumption}[Twice-differentiable objectives]
\label{assump:td-obj}
    The objective functions $f_i$ are at least twice differentiable.
\end{assumption}

\begin{assumption}[Decentralized communication topology]
\label{assump:topo}
    The graph is connected and undirected. $W$ is non-negative, symmetric, and doubly stochastic. 
\end{assumption}

\begin{assumption}[I.i.d. data distributions]
\label{assump:iid}
    All workers have access to the full dataset, so that
    \begin{equation*}
        f_1=f_2=\cdots=f_n=f
    \end{equation*}
\end{assumption}

\begin{assumption}[Martingale-difference gradient noise]
\label{assump:noise}
    Let $g_i^{(t)}$ be the mini-batch gradient on worker $i$. Under $f_i=f$, its first-order expansion around $\bar{x}^{(t-1)}$ is
    \begin{equation*}
        g_i^{(t)}
        =
        \nabla f(\bar{x}^{(t-1)})
        +
        H^{(t)}\delta_i^{(t-1)}
        +
        r_i^{(t)}
        +
        \xi_i^{(t)},
    \end{equation*}
    where $H^{(t)}=\nabla^2 f(\bar{x}^{(t-1)})$ is the Hessian, $r_i^{(t)}$ the Taylor residual, and $\xi_i^{(t)}$ mini-batch noise. In matrix form,
    \begin{equation*}
        G^{(t)}=\mathbf{1}\nabla f(\bar{x}^{(t-1)})^\top +  \Delta^{(t-1)}H^{(t)} + R^{(t)} + \Xi^{(t)}.
    \end{equation*}
    Let $\mathcal{F}_{t-1}$ be the information available before sampling the mini-batches at iteration $t$.
    The noise is conditionally mean zero:
    \begin{equation*}
        \mathbb{E}\left[\Xi^{(t)}\mid\mathcal{F}_{t-1}\right]=0
    \end{equation*}
\end{assumption}

\begin{assumption}[Frozen Hessian local window]
\label{assump:frozen-hessian}
    For the modal analysis, we study the local linearized dynamics in a short late-training window, setting $H^{(t)}=H_f$ and dropping the Taylor residual $R^{(t)}$.
    All modal variance and radius statements below refer to this local model.
\end{assumption}

\section{DSGD-AC and motivating observations}

This section motivates, develops, and analyzes DSGD-AC in three steps. We first show that consensus errors in DSGD vanish as the learning rate decays and argue that this limits its potential in generalization. We then propose DSGD-AC as a remedy and show empirically that it reaches flatter minima and higher test accuracy. Finally, we provide a theoretical analysis of the consensus-error dynamics, characterizing how the adaptive factor controls the magnitude and curvature profile of the maintained errors. All empirical claims are supported by training WRN16-8  \citep{zagoruyko2016wide} on CIFAR-100 \citep{krizhevsky2009learning} with 8 workers in a one-peer ring topology under cosine annealing with linear warm-up (see Figure~\ref{fig:d2c1}). Additional hyperparameter details are given in Appendix~\ref{app:hyper-details}.


\subsection{Vanishing consensus error in DSGD}

The conventional view of consensus errors as a nuisance was challenged by \citet{zhu2023decentralized}, who argued that one should view worker deviations $\delta_i^{(t)}$ as random perturbations around the global average $\bar{x}^{(t)}$, and proved that DSGD is asymptotically equivalent to average-direction SAM \citep{bisla2022low}. Under this interpretation, consensus errors provide implicit sharpness-aware regularization at no additional cost. However, the practical benefit of this observation is limited: unlike SAM, which maintains a fixed perturbation radius $\rho$, the consensus-error radius $\frac{1}{n}\sum_{i=1}^n\|\delta_i^{(t)}\|$ vanishes asymptotically as the learning rate decays \citep{lian2017can, zhu2022topology}. The implicit regularization thus disappears precisely when it is most needed, in the final stages of training where the model converges. This is confirmed empirically in Figure~\ref{fig:d2c2}.



To see why the consensus-error radius vanishes, it is instructive to view DSGD as stochastic gradient descent on the per-step surrogate $J^{(t)}$, defined as
\begin{equation}\label{eq:dsgd-actual-obj}
    \begin{aligned}
        J^{(t)}(x_1,\cdots,x_n)=\underbrace{\sum_{i=1}^nf_i(\bar{x}^{(t)})}_{\text{objective on deployed model}}+\underbrace{\sum_{i=1}^n [f_i(x_i^{(t)})-f_i(\bar{x}^{(t)})]}_{\text{sharpness}} +\underbrace{\frac{1}{4\alpha^{(t)}}\sum_{i,j\in[n]}W_{ij}\Vert x_i^{(t)}-x_j^{(t)}\Vert^2}_{\text{consensus regularizer}}
    \end{aligned}
\end{equation}

Differentiating $J^{(t)}$ with respect to $x_i$ and substituting back into the gradient descent update recovers the DSGD update in Eq.~(\ref{eq:dsgd-update}) exactly. The sharpness penalty measures the curvature seen by the workers: when the local models scatter away from $\bar{x}^{(t)}$, this term grows and penalizes sharp regions of the loss. As $\alpha^{(t)}\to 0$, however, the consensus regularizer weight $1/(4\alpha^{(t)})$ diverges and dominates $J^{(t)}$, forcing $x_i^{(t)}\to\bar{x}^{(t)}$ for all $i$. The sharpness penalty vanishes with the consensus errors, and the only remaining term is $\sum_i f_i(\bar{x}^{(t)})$, the same objective minimized by synchronous SGD. Maintaining a non-vanishing consensus-error radius throughout training is therefore necessary to preserve the potential regularization effect.


\begin{figure}[t]
    \centering
    \begin{subfigure}[t]{0.49\textwidth}
        \centering
        \includegraphics[width=0.9\linewidth]{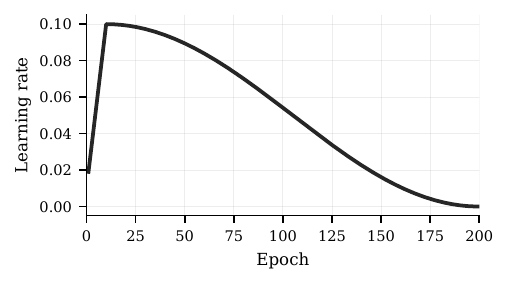}
        \subcaption{Cosine annealing learning rate schedule}
        \label{fig:d2c1}
    \end{subfigure}
    \hfill
    \begin{subfigure}[t]{0.49\linewidth}
        \centering
        \includegraphics[width=0.9\linewidth]{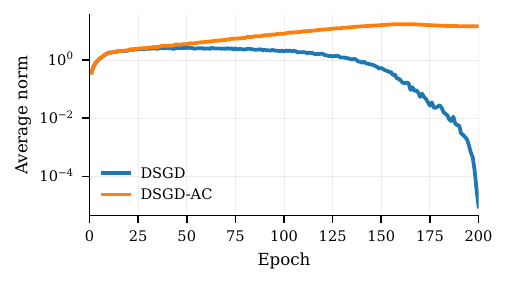}
        \subcaption{Norm of consensus errors evaluated across epochs}
        \label{fig:d2c2}
    \end{subfigure}
    \caption{Decentralized training of WRN16-8 on CIFAR-100 with 8 workers and the one-peer ring topology.}
\end{figure}

\begin{figure}[t]
    \centering
    \begin{subfigure}[t]{0.49\linewidth}
        \centering
        \includegraphics[width=0.9\linewidth]{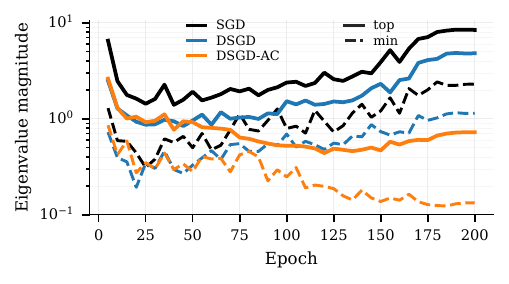}
        \subcaption{Hessian eigenvalue magnitudes measured by the Lanczos algorithm with 30 iterations}
        \label{fig:showcase-eigenvalues}
    \end{subfigure}
    \hfill
    \begin{subfigure}[t]{0.49\textwidth}
        \centering
        \includegraphics[width=0.9\linewidth]{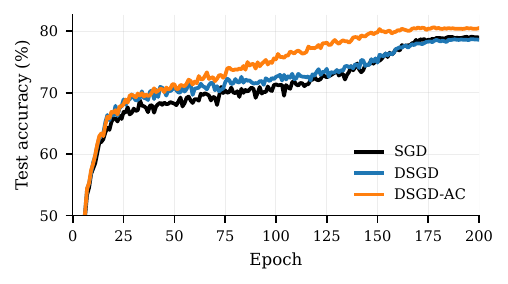}
        \subcaption{Test accuracy evaluated at the deployed models (global averaged center for decentralized training)}
        \label{fig:showcase-accuracy}
    \end{subfigure}
    \caption{Training performance and curvature for WRN16-8 on CIFAR-100 using 8 workers and the one-peer ring topology. DSGD-AC improves test accuracy and reaches a flatter local region than DSGD and synchronous SGD.}
    \label{fig:showcase-performance}
\end{figure}

\subsection{DSGD with adaptive consensus}

To maintain a non-vanishing perturbation radius throughout training, we propose  DSGD with adaptive consensus (DSGD-AC), detailed in Algorithm~\ref{alg:dsgdac}. The only modification to DSGD is a time-varying scalar factor $\gamma^{(t)}\in[0,1]$ that scales the consensus step at each iteration. In terms of the surrogate $J^{(t)}$ in Eq. (\ref{eq:dsgd-actual-obj}), this changes the consensus-regularizer weight from  $1/(4\alpha^{(t)})$ to $\gamma^{(t)}/(4\alpha^{(t)})$, preventing it from dominating as $\alpha^{(t)}\to 0$. The modified lines are highlighted in Algorithm~\ref{alg:dsgdac}, and the algorithm takes the global average $\bar{x}^{(t)}$ as the deployed model.


\begin{algorithm2e}
    \caption{Decentralized SGD with adaptive consensus (DSGD-AC) on worker $i$}
    \label{alg:dsgdac}
    \SetAlgoLined
    \KwData{Dataset ($D$), the number of workers ($N$), the number of epochs ($E$), the number of batches per epoch ($T$), initialization ($x^{(0)}$), and \colorbox{yellow!30}{a hyperparameter ($p\in\mathbb{R}^+$)}.}

    \KwResult{\colorbox{yellow!30}{Deployed model $\bar{x}=\frac{1}{n}\sum_{j=1}^n x_j^{(TE)}$}}

    $x_1^{(0)}=x_2^{(0)}=\cdots=x_n^{(0)}=x^{(0)}$

    \For{$t=1$ to $TE$} {
        $g_i^{(t)}=\nabla f(x_i^{(t-1)};s_i^{(t)})$

        \colorbox{yellow!30}{$\gamma^{(t)}=\left[\alpha^{(t)}/\alpha_{\max}\right]^p$}

        $x_i^{(t)}=x_i^{(t-1)}-\alpha^{(t)}g_i^{(t)}+$\colorbox{yellow!30}{$\gamma^{(t)}$}$ \sum_{j\in\mathcal{N}(i)} W_{ij}(x_j^{(t-1)}-x_i^{(t-1)})$
    }
\end{algorithm2e}

Note that $\alpha^{(t)}$ is determined by a learning rate scheduler such as cosine annealing \citep{loshchilov2016sgdr}, and $\alpha_{\max}$ is the maximal learning rate throughout the training, which ensures $\gamma^{(t)}$ is in the range $[0,1]$.
The schedule $\gamma^{(t)}=[\alpha^{(t)}/\alpha_{\max}]^p$ 
ties the consensus factor directly to the learning rate, so the two decay together. We use $p=3$ in the showcase experiments, which is chosen based on hyperparameter tuning.

As shown in Figure~\ref{fig:d2c2}, DSGD-AC keeps the consensus errors non-vanishing after the learning-rate decay, while also reaching a smaller top Hessian eigenvalue, a narrower Hessian spectrum, and improved test accuracy over DSGD and even synchronous SGD (Figures~\ref{fig:showcase-eigenvalues} and \ref{fig:showcase-accuracy}).
The experiments in Section~\ref{sec:exp} show that these gains persist across additional models, worker counts, and communication topologies.


\subsection{Empirical Hessian diagnostics}\label{sec:hessian-diagnostics}

We next investigate why maintaining larger disagreement does not simply degrade training. The loss envelope around the deployed average model $\bar{x}^{(t)}$ separates the effect of the worker deviations $\delta_i^{(t)}$ from the objective value at the average. A second-order expansion gives
\begin{equation}\label{eq:loss-envelope}
    \sum_{i=1}^n f(\bar{x}^{(t)}+\delta_i^{(t)})
    \approx
    n f(\bar{x}^{(t)})
    +
    \frac{1}{2}
    \sum_{i=1}^n(\delta_i^{(t)})^\top H^{(t)}\delta_i^{(t)},
\end{equation}
where the first-order terms cancel because $\sum_i\delta_i^{(t)}=0$, and the remainder is negligible when the consensus-error radius is small relative to the curvature scale of $f$. The quadratic term is a Hessian-weighted penalty on the consensus errors, concentrated on positive-curvature directions.

To assess this effect empirically, we track three quantities defined in terms of
the Hessian $H^{(t)}:=\nabla^2 f(\bar{x}^{(t)})$, its largest eigenvalue
$\lambda_1(H^{(t)})$, and the consensus errors $\Delta^{(t)}$.
The first diagnostic is the \emph{curvature exposure}, the Hessian-weighted consensus-error penalty from the loss envelope in Eq. (\ref{eq:loss-envelope}):
\begin{equation}\label{eq:consensus-alignment}
    Q_t
    :=
    \frac{\operatorname{tr}\!\left(\Delta^{(t)}H^{(t)}(\Delta^{(t)})^\top\right)}
         {\lambda_1(H^{(t)})}
    =
    \sum_{i=1}^n
    \frac{(\delta_i^{(t)})^\top H^{(t)}\delta_i^{(t)}}{\lambda_1(H^{(t)})}.
\end{equation}
The second is the \emph{normalized alignment score}, separating the
directional component of $Q_t$ from its radius: 
\begin{equation}\label{eq:normalized-consensus-alignment}
    A_t
    :=
    \frac{\operatorname{tr}\!\left(\widetilde{\Delta}^{(t)}H^{(t)}
    (\widetilde{\Delta}^{(t)})^\top\right)}{\lambda_1(H^{(t)})},
\end{equation}
Here, 
$\widetilde{\Delta}^{(t)}:=\Delta^{(t)}/\|\Delta^{(t)}\|_F$ so $Q_t = \|\Delta^{(t)}\|_F^2 \cdot A_t$, decomposing the penalty into
radius and alignment. The third is the \emph{random-direction baseline}, the
expected normalized alignment of Rademacher directions,
\begin{equation}\label{eq:random-direction-baseline}
    B_t
    :=
    \frac{\mathbb{E}_{u}\!\left[u^\top H^{(t)}u/\|u\|^2\right]}
         {\lambda_1(H^{(t)})}
    =
    \frac{\operatorname{tr}(H^{(t)})}{d\cdot\lambda_1(H^{(t)})},
\end{equation}
estimated using 50 Hutchinson trace probes. Together, $Q_t$, $A_t$, and $B_t$
diagnose how the consensus errors interact with the local curvature at the
deployed model; all three are normalized by $\lambda_1(H^{(t)})$ to remove
loss-scaling bias and facilitate comparison across runs. Further evaluation
details are in Appendix~\ref{appsec:eval-details}.



\begin{figure}[t]
    \centering
    \begin{subfigure}[t]{0.32\textwidth}
        \centering
        \includegraphics[width=\linewidth]{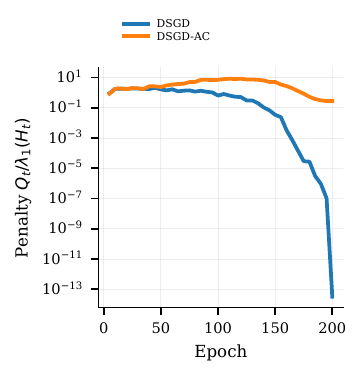}
        \subcaption{Curvature exposure $Q_t$.}
        \label{fig:vhv-penalty-panel}
    \end{subfigure}
    \hfill
    \begin{subfigure}[t]{0.32\linewidth}
        \centering
        \includegraphics[width=\linewidth]{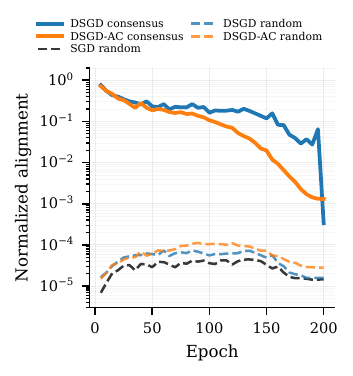}
        \subcaption{Alignment $A_t$ and baseline $B_t$.}
        \label{fig:vhv-alignment-panel}
    \end{subfigure}
    \hfill
    \begin{subfigure}[t]{0.32\linewidth}
        \centering
        \includegraphics[width=\linewidth]{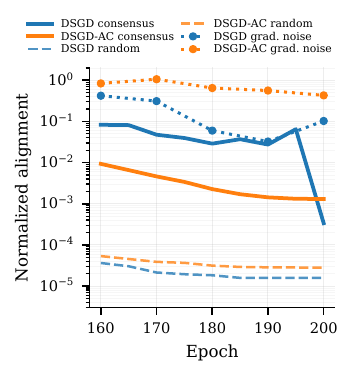}
        \subcaption{Gradient-noise alignment.}
        \label{fig:noise-alignment-panel}
    \end{subfigure}
    \caption{Consensus-error Hessian diagnostics for WRN16-8 on CIFAR-100: (a) normalized and unnormalized penalties, (b) consensus and random-direction alignment, and (c) late-training gradient-noise alignment. Except for $\mathcal{E}_t$, all quantities are normalized by the top Hessian eigenvalue.}
    \label{fig:vhv-diagnostics}
\end{figure}

The three-way comparison in Figure~\ref{fig:vhv-diagnostics} is instructive. In standard DSGD, the consensus-error radius collapses as $\alpha^{(t)}\to 0$, and $Q_t$ collapses with it despite maintained alignment $A_t$. $B_t$ shows that random directions have poor Hessian alignment, so their curvature exposure is low.
DSGD-AC occupies a different position: its normalized alignment $A_t$ is no stronger than DSGD, but, by preventing radius collapse, it maintains a substantially larger $Q_t$ than either alternative. The preference of DSGD-AC for flatter minima is therefore better understood as preservation of curvature exposure than as improvement of Hessian alignment.

This rules out two simpler explanations. DSGD-AC does not find flatter minima by aligning its consensus direction more strongly with the top Hessian eigenspace than DSGD. Nor do the consensus errors act as random noise: Figure~\ref{fig:vhv-alignment-panel} shows that $A_t$ remains well above the random baseline $B_t$ for both methods throughout late training. The maintained above-random alignment is instead consistent with the anisotropic structure of SGD noise in deep networks, as examined in the following section.

\subsection{Modal explanation of controlled disagreement}\label{sec:radius-control}

The empirical diagnostics in Section~\ref{sec:hessian-diagnostics} show that DSGD-AC preserves curvature exposure by maintaining a non-vanishing consensus-error radius. We now explain the mechanism theoretically. The central message is simple: decreasing $\gamma$ selectively preserves disagreement in low-curvature directions, while Hessian damping continues to suppress sharp directions. The analysis proceeds in three steps. Lemma~\ref{lemma:modal-dynamics} diagonalizes the disagreement dynamics into independent graph-Hessian modes. Theorem~\ref{theorem:modal-variance} gives each mode's stationary variance. Corollaries~\ref{corollary:radius-control} and \ref{corollary:gamma-schedule} translate these into radius predictions for $\gamma=\alpha^p$. Proofs are in Appendix~\ref{app:proof}.

We start from the matrix form of DSGD-AC, analogous to Eq. (\ref{eq:dsgd-update-matrix}):
\begin{equation}\label{eq:matrix-dsgdac}
    \begin{aligned}
        X^{(t)}&=X^{(t-1)}-\alpha^{(t)}G^{(t)}+\gamma^{(t)}(W-I)X^{(t-1)}\\
        &= (I-\gamma^{(t)}L)X^{(t-1)}-\alpha^{(t)}G^{(t)}
    \end{aligned}
\end{equation}

Multiplying by the center projector $P$ removes the average model and isolates the disagreement $\Delta^{(t)}=PX^{(t)}$. Since $W$ is doubly stochastic under Assumption~\ref{assump:topo},  $PL=LP$ and we find the disagreement recursion
\begin{equation}\label{eq:error-dynamic}
    \Delta^{(t)}= P(I-\gamma^{(t)}L)X^{(t-1)}-\alpha^{(t)}PG^{(t)} = (I-\gamma^{(t)}L)\Delta^{(t-1)}-\alpha^{(t)}PG^{(t)}
\end{equation}

Equation~(\ref{eq:error-dynamic}) already shows the two competing forces: mixing damps disagreement through $\gamma^{(t)}L$, while the worker-wise stochastic gradients continually inject disagreement through $PG^{(t)}$.
To see which Hessian directions receive this injected energy, we linearize the gradients around the deployed average model and diagonalize both the graph Laplacian and the Hessian.

\begin{lemma}[Linearized modal dynamics]\label{lemma:modal-dynamics}
    Under Assumptions~\ref{assump:td-obj}--\ref{assump:frozen-hessian}, the 
    local linearized dynamics in a late-training window decouple exactly into 
    independent scalar recursions. Let 
    $L=V\operatorname{diag}(\mu_1,\ldots,\mu_n)V^\top$ with 
    $2\geq\mu_1\geq\cdots\geq\mu_{n-1}>\mu_n=0$, and let 
    $H_f=U\operatorname{diag}(\lambda_1,\ldots,\lambda_d)U^\top$ with 
    $\lambda_1\geq\cdots\geq\lambda_d$. Define $z_{j,k}^{(t)}=v_j^\top\Delta^{(t)}u_k$ 
    and $\xi_{j,k}^{(t)}=v_j^\top\Xi^{(t)}u_k$. Then each graph-Hessian mode 
    satisfies
    \begin{equation}
        z_{j,k}^{(t)}
        =
        (1-\gamma^{(t)}\mu_j-\alpha^{(t)}\lambda_k)z_{j,k}^{(t-1)}
        -\alpha^{(t)}\xi_{j,k}^{(t)}.
    \end{equation}
\end{lemma}


The scalar recursion in Lemma~\ref{lemma:modal-dynamics} is the basic mechanism behind DSGD-AC, and explains many of the empirical observations in Section~\ref{sec:hessian-diagnostics}. The effective damping of mode $(j,k)$ is $d_{j,k}:=\gamma\mu_j+\alpha\lambda_k$, the sum of a graph term and a curvature term. Decreasing $\gamma$ weakens only the graph damping, while the positive Hessian curvature continues to suppress sharp directions independently. Theorem~\ref{theorem:modal-variance} makes this precise by computing the stationary variance of each mode as a function of $d_{j,k}$, the noise variance $\sigma_{j,k}^2$, and the learning rate $\alpha$.

\begin{theorem}[Stationary modal variance]\label{theorem:modal-variance}
    In the setting of Lemma~\ref{lemma:modal-dynamics}, fix constant $\alpha$ and $\gamma$. Assume the conditional modal noise variance is constant in the 
    local window:
    \begin{equation*}
        \sigma_{j,k}^2:=
        \mathbb{E}\left[\left(\xi_{j,k}^{(t)}\right)^2\mid\mathcal{F}_{t-1}\right].
    \end{equation*}
    Assume the stability condition $0<d_{j,k}<2$.
    When $d_{j,k}\leq 0$, the mode is unstable, and the variance diverges; this can occur for negative-curvature directions when $\alpha|\lambda_k|>\gamma\mu_j$. Under stability, for every disagreement mode $j\in\{1,\ldots,n-1\}$, the stationary modal variance is
    \begin{equation}\label{eq:stationary-modal-variance}
        \mathbb{E}[z_{j,k}^2]
        =
        \frac{\alpha^2\sigma_{j,k}^2}
        {d_{j,k}(2-d_{j,k})}.
    \end{equation}
\end{theorem}

Theorem~\ref{theorem:modal-variance} gives the exact stationary variance for fixed $\alpha$ and $\gamma$. Remark~\ref{remark:modal-orders} derives order-of-magnitude bounds that are more useful in practice, separating two regimes: modes where graph damping dominates (low curvature) and modes where Hessian damping dominates (high curvature). These bounds are the key tool for understanding how $\gamma$ shapes the radius.

\begin{remark}[Modal scaling bounds]\label{remark:modal-orders}
    Suppose the conditions of Theorem~\ref{theorem:modal-variance} hold and $d_{j,k}\leq d_{\max}<2$.
    Then
    \begin{equation}
        \frac{\alpha^2\sigma_{j,k}^2}{2d_{j,k}}
        \leq
        \mathbb{E}[z_{j,k}^2]
        \leq
        \frac{\alpha^2\sigma_{j,k}^2}{(2-d_{\max})d_{j,k}}.
    \end{equation}
    If a mode is mixing-dominated, meaning $|\alpha\lambda_k|\leq\rho\gamma\mu_j$ for a fixed $\rho\in[0,1)$, then
    \begin{equation}
        \mathbb{E}[z_{j,k}^2]
        =
        \Theta\left(\frac{\alpha^2\sigma_{j,k}^2}{\gamma\mu_j}\right).
    \end{equation}
    If $\lambda_k>0$, then
    \begin{equation}
        \mathbb{E}[z_{j,k}^2]
        =
        \Theta\left(\frac{\alpha^2\sigma_{j,k}^2}{\gamma\mu_j+\alpha\lambda_k}\right),
    \end{equation}
    so positive high-curvature directions are damped by the Hessian term. For negative-curvature directions, the same formula holds whenever $0<d_{j,k}<2$, but the Hessian term reduces $d_{j,k}$ and can increase the variance.
\end{remark}

The remark shows that low-curvature modes scale as $\Theta(\alpha^2/\gamma)$ while high-curvature modes are suppressed to $\mathcal{O}(\alpha)$. Corollary~\ref{corollary:radius-control} sums these contributions across modes to characterize the total disagreement radius, and identifies the condition under which the low-curvature contribution dominates.

\begin{corollary}[Conditional radius scaling from low-curvature modes]\label{corollary:radius-control}
    Let $\mathcal{S}_{\mathrm{low}}=\{(j,k):1\leq j\leq n-1,\ 
    |\alpha\lambda_k|\leq\rho\gamma\mu_j\}$ for some fixed $\rho\in[0,1)$, and 
    let $\mathcal{S}_{\mathrm{high}}^+=\{(j,k):1\leq j\leq n-1,\ \lambda_k\geq 
    c>0\}$ for a fixed curvature threshold $c$. The set $\mathcal{S}_\mathrm{low}$ 
    is non-empty only when genuinely flat directions exist, i.e., directions with 
    $|\lambda_k|\ll\gamma\mu_j/\alpha$. Define the active low-curvature modal 
    weight
    \begin{equation}
        M_{\mathrm{low}}(\alpha,\gamma)
        :=
        \sum_{(j,k)\in\mathcal{S}_{\mathrm{low}}}\frac{\sigma_{j,k}^2}{\mu_j}.
    \end{equation}
    If $M_{\mathrm{low}}(\alpha,\gamma)>0$ and the modes in 
    $\mathcal{S}_{\mathrm{low}}$ are stable, then
    \begin{equation}
        \sum_{(j,k)\in\mathcal{S}_{\mathrm{low}}}\mathbb{E}[z_{j,k}^2]
        =
        \Theta\left(\frac{\alpha^2}{\gamma}M_{\mathrm{low}}(\alpha,\gamma)\right).
    \end{equation}
    If $M_{\mathrm{low}}(\alpha,\gamma)$ is bounded above and below by positive 
    constants, this is $\Theta(\alpha^2/\gamma)$. If 
    $\sigma_{j,k}^2\leq\sigma_+^2<\infty$ on $\mathcal{S}_{\mathrm{high}}^+$ 
    and the corresponding modes are stable, then
    \begin{equation}
        \sum_{(j,k)\in\mathcal{S}_{\mathrm{high}}^+}\mathbb{E}[z_{j,k}^2]
        =
        O(\alpha).
    \end{equation}
    The low-curvature contribution dominates when $\alpha^2/\gamma=\Omega(\alpha)$, 
    i.e. when $\gamma=O(\alpha)$. For $\gamma=\alpha^p$, this holds for all 
    $p\geq1$; the practically relevant regime $p\geq2$ is analyzed in 
    Corollary~\ref{corollary:gamma-schedule}.
\end{corollary}


The set $\mathcal{S}_{\mathrm{low}}$ is non-empty whenever $H_f$ has a
nontrivial null space, since any almost-flat direction satisfies
$|\alpha\lambda_k|\leq \rho\gamma\mu_j$. More generally, low-rank Hessian
structure in neural networks motivates the presence of such modes \citep{keskar2016large, singh2021analytic, song2024does, ben2024high}.

The radius scaling in Corollary~\ref{corollary:radius-control} is stated for 
general $\gamma$. Corollary~\ref{corollary:gamma-schedule} specializes to the 
schedule $\gamma=\alpha^p$ used in DSGD-AC, which gives a direct prediction 
for how the maintained radius depends on $p$ as $\alpha$ decays.

\begin{corollary}[Conditional scaling under $\gamma=\alpha^p$]\label{corollary:gamma-schedule}
    Set $\gamma=\alpha^p$ up to a fixed constant factor, with $p\geq2$. For mixing-dominated modes, the active condition $|\lambda_k|\leq\rho\mu_j\alpha^{p-1}$ tightens as $\alpha\to0$, so the following applies only to modes that remain active. For such modes,
    \begin{equation}
        \mathbb{E}[z_{j,k}^2]
        =
        \Theta\left(\alpha^{2-p}\sigma_{j,k}^2/\mu_j\right).
    \end{equation}
    For $p=2$, the variance is of constant order in each active mode. For $p>2$, it grows as $\alpha\to0$; if unchecked, this can destabilize training, consistent with the performance degradation observed for $p>4$ in Table~\ref{tab:sensitivity_p}.
\end{corollary}

These results show the mechanism behind DSGD-AC.
Reducing $\gamma$ preserves disagreement magnitude in active low-curvature modes, while the Hessian term $\alpha\lambda_k$ continues to suppress positive high-curvature modes.
For $\gamma=\alpha^p$, $p=2$ maintains constant-order variance in surviving mixing-dominated modes, whereas $p>2$ increases their linearized stationary variance as $\alpha$ decays.
Because the active condition tightens over training, the maintained radius is filtered toward flatter directions rather than spread uniformly across the spectrum.

The curvature exposure $Q_t$ introduced in Section~\ref{sec:hessian-diagnostics} can be written in modal coordinates as
\begin{equation}
    Q_t
    =
    \sum_{i=1}^n(\delta_i^{(t)})^\top H_f\delta_i^{(t)}/\lambda_1(H_f)
    =
    \sum_{k=1}^d\lambda_k\sum_{j=1}^{n-1} (z_{j,k}^{(t)})^2 /\lambda_1(H_f).
\end{equation}
This decomposition separates the two effects identified empirically: the magnitude of the maintained disagreement, controlled by $\gamma$, and its distribution across Hessian eigen directions, controlled by the modal noise variances and curvature-dependent damping $d_{j,k}$. A method with weaker normalized Hessian alignment can still produce larger $Q_t$ if its radius is sufficiently larger — precisely what DSGD-AC achieves.

\begin{remark}[Source-noise anisotropy and Hessian alignment]\label{remark:noise-alignment}
    The modal formula in Remark~\ref{remark:modal-orders} implicitly explains why the above-random alignment $A_t > B_t$ observed in Figure~\ref{fig:vhv-alignment-panel} should not be attributed to adaptive consensus. If the injected noise $\sigma_{j,k}^2$ were isotropic in the Hessian eigen basis, the curvature term in $d_{j,k}=\gamma\mu_j+\alpha\lambda_k$ would damp high-curvature modes more strongly, and alignment would not exceed the random baseline.
    The observed alignment is instead consistent with anisotropic source noise: SGD noise in deep networks is known to concentrate energy in sharp directions~\citep{wu2022alignment,ziyin2022strength}, and Figure~\ref{fig:noise-alignment-panel} verifies this directly for late-training gradient noise. Worker disagreements in both DSGD and DSGD-AC are driven by this structured noise, and graph mixing and curvature-dependent damping determine how much of each injected mode survives.
\end{remark}

\section{Numerical Experiments}\label{sec:exp}

In this section, we present extensive numerical experiments on image classification using Wide ResNets \citep{zagoruyko2016wide} trained on CIFAR-10 and CIFAR-100 \citep{krizhevsky2009learning} under a variety of distributed setups. We follow the hyperparameter settings from the original papers and reproduce comparable baseline performance to ensure fair comparisons.
All models are trained for 200 epochs.
The global batch size is set to 128 for 8 workers and is linearly scaled with the number of workers; mixed-precision training is enabled by default.
Since the networks use batch normalization layers and the training was not explicitly done on the global average center, we perform a calibration pass over the training set before evaluation, following \cite{defazio2024road}\footnote{\url{https://github.com/facebookresearch/schedule_free}}, which yields more reliable test performance.
We report mean $\pm$ standard deviation from 3 repeat runs, and shaded regions in the plots indicate 95\% confidence intervals. 
All experiments extend the decentralized training framework of \citet{wang2025promise},\footnote{\url{https://github.com/WangZesen/Decent-DP}}, which overlaps decentralized communication with computation to hide communication latency. Additional training and evaluation details can be found in Appendix~\ref{app:hyper-details}.


This work focuses on DSGD with adaptive consensus, where the base optimizer is SGD without preconditioning. Appendix~\ref{app:dadam-exp} reports additional experiments with Adam~\citep{kingma2014adam} as the base optimizer for Transformer models~\citep{vaswani2017attention} on WMT-14~\citep{bojar2014findings}, showing that adaptive consensus extends to preconditioned optimizers.


\subsection{Ablation study and sensitivity analysis}\label{sec:sensitivity}

\begin{table}
\begin{minipage}[t]{0.47\textwidth}
\centering
\resizebox{\linewidth}{!}{%
{
\begin{tabular}{c c c c}
\toprule
$p$ & \textbf{Test Acc.} (\%) $\uparrow$
    & \textbf{Train Loss} $\downarrow$
    & \textbf{Test Loss} $\downarrow$ \\
\midrule
0 & 96.04 {\tiny± 0.21} & \uline{0.0006} {\tiny± 0.0000} & 0.182 {\tiny± 0.007} \\
1 & 96.36 {\tiny± 0.06} & 0.0007 {\tiny± 0.0000} & 0.160 {\tiny± 0.002} \\
2 & 96.48 {\tiny± 0.13} & 0.0017 {\tiny± 0.0001} & 0.150 {\tiny± 0.002} \\
3 & \uline{96.55} {\tiny± 0.10} & 0.0163 {\tiny± 0.0005} & 0.122 {\tiny± 0.002} \\
4 & 96.41 {\tiny± 0.04} & 0.0347 {\tiny± 0.0010} & \uline{0.118} {\tiny± 0.001} \\
5 & 96.19 {\tiny± 0.05} & 0.0475 {\tiny± 0.0013} & 0.121 {\tiny± 0.002} \\
\bottomrule
\end{tabular}%
}}
\captionof{table}{Sensitivity analysis of parameter $p$ in the WRN28-10 on CIFAR-10 experiments with 16 workers, one-peer ring topology, and $E_{\text{start}}=100$. The best metric values are \uline{underlined}.}
\label{tab:sensitivity_p}
\end{minipage}
\hfill
\begin{minipage}[t]{0.505\textwidth}
\centering
\resizebox{\linewidth}{!}{%
{
\begin{tabular}{c c c c}
\toprule
$E_{\text{start}}$ & \textbf{Test Acc.} (\%) $\uparrow$
    & \textbf{Train Loss} $\downarrow$
    & \textbf{Test Loss} $\downarrow$ \\
\midrule
10  & 96.16 {\tiny± 0.06} & 0.0474 {\tiny± 0.0011} & 0.121 {\tiny± 0.003} \\
50  & 96.31 {\tiny± 0.12} & 0.0395 {\tiny± 0.0006} & 0.121 {\tiny± 0.003} \\
75  & 96.33 {\tiny± 0.09} & 0.0290 {\tiny± 0.0005} & \uline{0.119} {\tiny± 0.002} \\
100 & \uline{96.55} {\tiny± 0.10} & 0.0163 {\tiny± 0.0005} & 0.122 {\tiny± 0.002} \\
150 & 96.49 {\tiny± 0.08} & 0.0021 {\tiny± 0.0003} & 0.150 {\tiny± 0.002} \\
175 & 96.21 {\tiny± 0.12} & \uline{0.0008} {\tiny± 0.0001} & 0.181 {\tiny± 0.003} \\
\bottomrule
\end{tabular}%
}}
\captionof{table}{Sensitivity analysis of $E_{\text{start}}$ in the WRN28-10 on CIFAR-10 experiments with 16 workers, one-peer ring topology, and $p=3$. The best metric values are \uline{underlined}.}
\label{tab:start_epoch_tune}
\end{minipage}

\end{table}

\begin{figure}[t]
    \centering
    \includegraphics[width=0.45\linewidth]{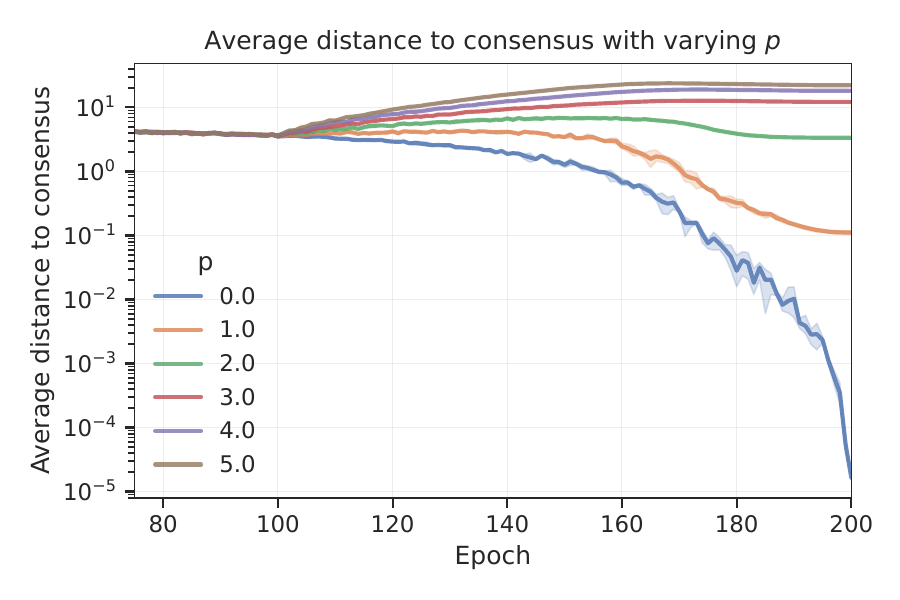}
    \caption{Average norm of consensus errors over epochs with varying $p$ in the WRN28-10 on CIFAR-10 experiments with 16 workers, the one-peer ring topology, and $E_{\text{start}}=100$.}
    \label{fig:sensitivity_p_d2c}
\end{figure}

\paragraph{Start epoch $E_\text{start}$}
We introduce $E_\text{start}$ as the epoch from which the adaptive consensus mechanism is activated. Before $E_\text{start}$, $\gamma$ is set to~1 and DSGD-AC reduces to standard DSGD. In setups with more workers or weaker connectivity, decentralization already induces substantial consensus errors in early training, and reducing $\gamma$ below~1 at this stage further weakens mixing, hindering loss optimization before the model has converged to a stable region. Activating too late, on the other hand, leaves too few epochs for the regularisation to act. Table~\ref{tab:start_epoch_tune} shows this two-sided sensitivity: for 16 workers on the one-peer ring topology, $E_\text{start}=100$ gives the best test accuracy, with performance degrading at both extremes. As a practical heuristic, $E_\text{start}$ should follow the warm-up phase and scale with worker count and topology: more workers and weaker connectivity require a later activation.

\paragraph{Parameter $p$}
Table~\ref{tab:sensitivity_p} presents the sensitivity analysis with respect to $p$. The results indicate that the best test accuracy and test loss are achieved for $p\in[2,4]$.
Since DSGD-AC reduces to DSGD when $p=0$, these experiments also serve as an ablation study of the adaptive consensus mechanism.
With $p>0$, DSGD-AC always achieves better performance than DSGD. Figure~\ref{fig:sensitivity_p_d2c} presents the average norm of the consensus errors with various $p$. 
Consistent with the analysis in Section~\ref{sec:radius-control}, the consensus errors persist for $p\geq 2$, and the regularization and generalization improvement become weak with a too-large $p$ because of worse alignment.


Appendix~\ref{appsec:sensitivity} includes the training curves of experiments in the sensitivity analysis.

\subsection{Tuned results with various setups}

\begin{table}[t]
\centering
\resizebox{\linewidth}{!}{
\begin{tblr}{
  cells = {c},
  cell{1}{1} = {r=2}{},
  cell{1}{2} = {r=2}{},
  cell{1}{3} = {r=2}{},
  cell{1}{4} = {c=3}{},
  cell{3}{1} = {r=9}{},
  cell{3}{2} = {r=3}{},
  cell{3}{4} = {r=3}{},
  cell{6}{2} = {r=3}{},
  cell{6}{4} = {r=3}{},
  cell{12}{1} = {r=9}{},
  cell{9}{2} = {r=3}{},
  cell{9}{4} = {r=3}{},
  cell{12}{2} = {r=3}{},
  cell{12}{4} = {r=3}{},
  cell{15}{2} = {r=3}{},
  cell{15}{4} = {r=3}{},
  cell{18}{2} = {r=3}{},
  cell{18}{4} = {r=3}{},
  hline{1,21} = {-}{0.10em},
  hline{2} = {4-6}{0.02em},
  hline{3,12} = {-}{0.06em},
  hline{6,9,15,18} = {2-6}{0.02em},
}
\textbf{Model}     & \textbf{\#W} & \textbf{Topo.} & \textbf{Algorithms} (\textbf{Test Accuracy} (\%) $\uparrow$ / \textbf{Test Loss} $\downarrow$ / \textbf{Top-1 Eigenvalue} $\downarrow$)                   &                              &                              \\
                     &         &          & Synchronous SGD                          & Decentralized SGD                         & DSGD-AC (ours)                      \\
28-10 & 8       & ring     & 80.00 {\tiny± 0.19} / 1.097 {\tiny± 0.014} / 4.53 {\tiny± 0.72} & 79.71 {\tiny± 0.13} / 1.108 {\tiny± 0.011} / 4.35 {\tiny± 1.19} & \uline{82.12} {\tiny± 0.12} / \uline{0.897} {\tiny± 0.007} / 0.59 {\tiny± 0.02} \\
                     &         & exp      &                              & 79.98 {\tiny± 0.33} / 1.100 {\tiny± 0.018} / 4.60 {\tiny± 1.20} & 81.82 {\tiny± 0.08} / 0.911 {\tiny± 0.013} / 0.53 {\tiny± 0.04} \\
                     &         & complete &                              & 79.62 {\tiny± 0.05} / 1.114 {\tiny± 0.014} / 4.13 {\tiny± 0.85} & 81.41 {\tiny± 0.31} / 0.953 {\tiny± 0.001} / \uline{0.47} {\tiny± 0.09} \\
                     & 16      & ring     & 79.79 {\tiny± 0.25} / 1.111 {\tiny± 0.021} / 6.75 {\tiny± 0.97} & 80.18 {\tiny± 0.14} / 1.050 {\tiny± 0.015} / 2.89 {\tiny± 1.01} & \uline{82.19} {\tiny± 0.17} / \uline{0.846} {\tiny± 0.012} / 0.51 {\tiny± 0.02} \\
                     &         & exp      &                              & 80.18 {\tiny± 0.28} / 1.066 {\tiny± 0.003} / 3.31 {\tiny± 0.82} & 82.09 {\tiny± 0.03} / 0.886 {\tiny± 0.003} / 0.57 {\tiny± 0.16} \\
                     &         & complete &                              & 79.81 {\tiny± 0.29} / 1.094 {\tiny± 0.007} / 4.11 {\tiny± 0.40} & 81.64 {\tiny± 0.10} / 0.934 {\tiny± 0.010} / \uline{0.45} {\tiny± 0.03} \\
                     & 32      & ring     & 79.90 {\tiny± 0.19} / 1.090 {\tiny± 0.019} / 4.02 {\tiny± 0.69} & 80.72 {\tiny± 0.34} / 0.950 {\tiny± 0.019} / 4.15 {\tiny± 0.61} & \uline{81.92} {\tiny± 0.13} / \uline{0.831} {\tiny± 0.006} / 0.62 {\tiny± 0.12} \\
                     &         & exp      &                              & 80.56 {\tiny± 0.21} / 1.000 {\tiny± 0.006} / 3.58 {\tiny± 1.06} & \uline{81.92} {\tiny± 0.04} / 0.886 {\tiny± 0.011} / \uline{0.37} {\tiny± 0.04} \\
                     &         & complete &                              & 80.04 {\tiny± 0.12} / 1.060 {\tiny± 0.011} / 3.50 {\tiny± 1.09} & 81.58 {\tiny± 0.15} / 0.925 {\tiny± 0.001} / 0.42 {\tiny± 0.09} \\
16-8  & 8       & ring     & 79.19 {\tiny± 0.13} / 0.997 {\tiny± 0.007} / 6.81 {\tiny± 2.34} & 79.16 {\tiny± 0.36} / 0.986 {\tiny± 0.010} / 3.51 {\tiny± 0.60} & \uline{80.54} {\tiny± 0.31} / 0.910 {\tiny± 0.016} / 0.76 {\tiny± 0.03} \\
                     &         & exp      &                              & 78.90 {\tiny± 0.14} / 1.002 {\tiny± 0.012} / 4.71 {\tiny± 0.33} & 80.51 {\tiny± 0.08} / \uline{0.902} {\tiny± 0.004} / \uline{0.75} {\tiny± 0.04} \\
                     &         & complete &                              & 79.12 {\tiny± 0.22} / 0.989 {\tiny± 0.012} / 8.89 {\tiny± 1.88} & 80.08 {\tiny± 0.20} / 0.931 {\tiny± 0.002} / 1.18 {\tiny± 0.76} \\
                     & 16      & ring     & 79.02 {\tiny± 0.27} / 1.012 {\tiny± 0.014} / 6.32 {\tiny± 1.21} & 79.07 {\tiny± 0.22} / 0.979 {\tiny± 0.008} / 4.01 {\tiny± 1.36} & \uline{80.51} {\tiny± 0.11} / 0.861 {\tiny± 0.015} / 0.56 {\tiny± 0.01} \\
                     &         & exp      &                              & 79.02 {\tiny± 0.14} / 0.982 {\tiny± 0.011} / 4.91 {\tiny± 0.37} & 80.44 {\tiny± 0.08} / \uline{0.889} {\tiny± 0.006} / 0.57 {\tiny± 0.02} \\
                     &         & complete &                              & 79.03 {\tiny± 0.22} / 1.001 {\tiny± 0.012} / 5.32 {\tiny± 2.60} & 80.01 {\tiny± 0.06} / 0.919 {\tiny± 0.016} / \uline{0.51} {\tiny± 0.01} \\
                     & 32      & ring     & 78.90 {\tiny± 0.14} / 1.010 {\tiny± 0.009} / 8.41 {\tiny± 1.13} & 79.12 {\tiny± 0.24} / 0.976 {\tiny± 0.006} / 2.95 {\tiny± 0.39} & 80.17 {\tiny± 0.10} / \uline{0.872} {\tiny± 0.005} / 0.72 {\tiny± 0.21}\\
                     &         & exp      &                              & 79.13 {\tiny± 0.04} / 0.974 {\tiny± 0.004} / 3.17 {\tiny± 0.19} & \uline{80.22} {\tiny± 0.16} / 0.921 {\tiny± 0.009} / \uline{0.51} {\tiny± 0.03} \\
                     &         & complete &                              & 79.04 {\tiny± 0.19} / 0.988 {\tiny± 0.014} / 2.81 {\tiny± 0.36} & 79.73 {\tiny± 0.13} / 0.955 {\tiny± 0.006} / 0.52 {\tiny± 0.01}
\end{tblr}}
\caption{Performance comparison of synchronous SGD, decentralized SGD, and DSGD-AC. Dataset: CIFAR-100. Model: WRN28-10 and WRN16-8. Number of workers: 8, 16, 32. Topology: one-peer ring, one-peer exponential graph \citep{ying2021exponential}, and complete graph. $p=3$ for all DSGD-AC experiments. $E_\text{start}$ is set to 10, 100 and 150 for 8-worker, 16-worker and 32-worker setups, respectively. The best metric values for each distributed setup are \uline{underlined}.}\label{tab:algo_comparison}
\end{table}

In Table~\ref{tab:algo_comparison}, we summarize the experiment results of DSGD-AC with tuned $p$ and $E_\text{start}$ based on Tables~\ref{tab:sensitivity_p} and \ref{tab:start_epoch_tune}, and the comparisons with SGD and DSGD in various setups. The algorithms are applied on training WRN28-10 and WRN16-8 on the CIFAR-100 dataset. The number of workers varies from 8 to 32, and the topologies are one-peer ring, one-peer exponential, and complete graph. According to the sensitivity analysis, we choose $p=3$ for all DSGD-AC experiments, and $E_\text{start}=$10, 100, and 150 for 8-worker, 16-worker, and 32-worker setups, respectively.
Since the best sharpness metric remains an open question, we report the top-1 Hessian eigenvalue as a
surrogate, which has been empirically shown to correlate strongly with generalization~\citep{bisla2022low, luo2024explicit}.

Across all evaluated setups, DSGD-AC achieves consistently higher test accuracy and lower test loss than both synchronous SGD and decentralized SGD (even under their optimal setups).
These results validate the effectiveness and robustness of DSGD-AC across models, topologies, and system scales.
The corresponding training curves can be found in Appendix~\ref{appsec:curves}. Finally, we validate that DSGD-AC specifically targets flatter minima as DSGD-AC significantly reduces the dominant Hessian eigenvalue compared to standard DSGD and synchronous SGD. This confirms that the adaptive consensus errors provide implicit sharpness-aware regularization similar to SAM \citep{foret2020sharpness}, but without the $2\times$ computational overhead required by explicit sharpness minimization methods (see Appendix~\ref{appsec:sam} for details).

\subsection{Training efficiency}
Table~\ref{tab:train-time} shows the training time with WRN28-10 and CIFAR-100 using 1, 2, and 4 $8\times$T4 nodes interconnected by 100Gbps InfiniBand. Thanks to the decentralized communication pattern, decentralized training can hide the communication time and provide more resilience to stragglers, which makes it more efficient than synchronous training even under the single-node setup \citep{lian2017can, assran2019stochastic, wang2025promise}.
In these experiments, decentralized methods require only $0.78$–$0.85\times$ the training time of synchronous SGD, and the relative speedup further improves as the number of workers increases. Compared with DSGD, DSGD-AC introduces negligible additional runtime overhead.

\begin{table}[t]
\centering
\resizebox{0.7\linewidth}{!}{%
{
\begin{tabular}{c c c c}
\toprule
\textbf{\# Nodes} & \textbf{SGD} & \textbf{DSGD} & \textbf{DSGD-AC} \\
\midrule
1 & 91.50 ± 0.21 ($\times 1.000$) & 77.69 ± 0.45 ($\times 0.849$) & 78.11 ± 0.52 ($\times 0.854$) \\
2 & 48.38 ± 0.11 ($\times 1.000$) & 39.09 ± 0.19 ($\times 0.808$) & 39.69 ± 0.14 ($\times 0.820$) \\
4 & 25.78 ± 0.07 ($\times 1.000$) & 20.08 ± 0.12 ($\times 0.779$) & 20.28 ± 0.14 ($\times 0.787$) \\
\bottomrule
\end{tabular}%
}}
\caption{Training time (in minutes) of SGD, DSGD, and DSGD-AC in the WRN28-10 on CIFAR-100 experiments. The communication topology for decentralized methods is the one-peer ring topology. $\times x$ is the relative time compared with synchronous SGD in the same distributed setup.}\label{tab:train-time}
\end{table}

\section{Conclusion and discussion}\label{sec:conclusion}

Prior work on decentralized training has focused on reducing consensus errors, viewing them as an obstacle to convergence and generalization. This work argues the opposite: consensus errors, maintained at a controlled level, act as an implicit regularizer and improve generalization over both standard DSGD and synchronous SGD at negligible additional cost. The key theoretical contribution is a modal analysis of the consensus-error dynamics, which characterizes how the adaptive factor $\gamma^{(t)}$ balances radius and curvature profile through the stationary variance of graph-Hessian modes. This analysis introduces curvature exposure as the quantity that governs the regularization effect, and shows that DSGD-AC preserves it better than standard DSGD. Together with the scalability inherited from decentralized training, it positions DSGD-AC as a compelling alternative to synchronous SGD for distributed training.


\paragraph{Limitation and future work} Two limitations merit attention. First, DSGD-AC achieves a weaker normalized Hessian alignment than standard DSGD, as shown in Figure~\ref{fig:vhv-diagnostics}. The analysis in Section~\ref{sec:hessian-diagnostics} shows this does not undermine the curvature exposure argument, since the larger radius compensates. However, it suggests that the linear mixing operator with a decaying scalar factor is not optimal: a mixing operator designed to preserve alignment while maintaining radius could improve the regularization effect further. Second, the mechanism relies on the anisotropy of SGD gradient noise, which aligns naturally with the Hessian eigen space. Adaptive optimizers such as Adam \citep{kingma2014adam} modulate this noise, disrupting the anisotropy and weakening the alignment of consensus errors with sharp directions. Extending adaptive consensus to such optimizers is therefore non-trivial and remains an open problem. Whether DSGD-AC finds different basins of attraction than DSGD, and what role escape dynamics play in the observed generalization improvement, is a further open question.

\subsubsection*{Broader Impact Statement}
The proposed method may contribute to more efficient distributed training by improving the generalization performance of decentralized optimization without requiring extra gradient evaluations or substantial memory overhead. Such efficiency gains could reduce communication costs and, in some settings, lower the computational resources needed to reach a target accuracy. However, efficiency improvements may also encourage larger-scale training runs, potentially offsetting environmental benefits. We therefore view DSGD-AC as a tool for improving the quality-efficiency trade-off, rather than as a standalone solution to the environmental costs of modern deep learning.

\bibliography{main}
\bibliographystyle{arxiv_natbib}

\appendix
\section{Appendix}

\subsection{Proofs and derivations}\label{app:proof}

\subsubsection{Proof of Lemma~\ref{lemma:modal-dynamics}}

\textit{Proof.} 
Since $P\mathbf{1}=0$ and $P\Delta^{(t-1)}=\Delta^{(t-1)}$ (as $\sum_i\delta_i^{(t-1)}=0$ by the 
definition of $\Delta$), applying Assumption~\ref{assump:noise} to the disagreement
recursion Equation~\ref{eq:error-dynamic} and invoking Assumption~\ref{assump:frozen-hessian}
($H^{(t)}=H_f$, $R^{(t)}=0$) gives
%
\begin{equation}\label{eq:lemma1-1}
    \begin{aligned}
        \Delta^{(t)} &= (I-\gamma^{(t)}L)\Delta^{(t-1)}-\alpha^{(t)}PG^{(t)} \\
        &= (I-\gamma^{(t)}L)\Delta^{(t-1)}- \alpha^{(t)} P (\mathbf{1}\nabla f(\bar{x}^{(t-1)})^\top + \Delta^{(t-1)}H^{(t)}+R^{(t)}+\Xi^{(t)}) \\
        &= (I-\gamma^{(t)}L)\Delta^{(t-1)}- \alpha^{(t)} P (\Delta^{(t-1)}H^{(t)} + R^{(t)}+\Xi^{(t)}) \\
        &= (I-\gamma^{(t)}L)\Delta^{(t-1)} -\alpha^{(t)} \Delta^{(t-1)}H_f - \alpha^{(t)} P \Xi^{(t)}.
    \end{aligned}
\end{equation}
Multiplying Eq.~(\ref{eq:lemma1-1}) by $u_k$ on the right and using $H_f u_k=\lambda_k u_k$ gives
\begin{equation}
    \Delta^{(t)}u_k
    =
    (I-\gamma^{(t)}L-\alpha^{(t)}\lambda_k I)\Delta^{(t-1)}u_k
    -\alpha^{(t)}P\Xi^{(t)}u_k.
\end{equation}
For every disagreement mode $j\in\{1,\ldots,n-1\}$, $Pv_j=v_j$ and $Lv_j=\mu_jv_j$.
Multiplying by $v_j^\top$ on the left yields 
and setting
$z_{j,k}^{(t)}:=v_j^\top\Delta^{(t)}u_k$, $\xi_{j,k}^{(t)}:=v_j^\top\Xi^{(t)}u_k$
yields the claimed recursion.
\qed

\subsubsection{Proof of Theorem~\ref{theorem:modal-variance}}

\textit{Proof.}
For fixed $\alpha$ and $\gamma$, Lemma~\ref{lemma:modal-dynamics} gives
\begin{equation}
    z_{j,k}^{(t)}
    =
    a_{j,k}z_{j,k}^{(t-1)}-\alpha\xi_{j,k}^{(t)},
    \quad
    a_{j,k}:=1-d_{j,k}
\end{equation}
Since $z_{j,k}^{(t-1)}$ is $\mathcal{F}_{t-1}$-measurable and
$\mathbb{E}[\xi_{j,k}^{(t)}\mid\mathcal{F}_{t-1}]=0$, squaring and taking 
conditional expectation yields
\begin{equation}
    \mathbb{E}[(z_{j,k}^{(t)})^2\mid\mathcal{F}_{t-1}]
    =
    a_{j,k}^2(z_{j,k}^{(t-1)})^2
    +
    \alpha^2\sigma_{j,k}^2.
\end{equation}
The condition $0<d_{j,k}<2$ is equivalent to $|a_{j,k}|<1$, so the scalar autoregressive process is stable and admits a unique stationary distribution with finite second moment. In particular, imposing stationarity  
$\mathbb{E}[(z_{j,k}^{(t)})^2]=\mathbb{E}[(z_{j,k}^{(t-1)})^2]=V$ implies that 
\begin{equation}
    V = a_{j,k}^2\,V + \alpha^2\sigma_{j,k}^2.
\end{equation}
Since $\vert a_{j,k}\vert <1$  implies $1-a_{j,k}^2\neq 0$,  
this equation has the unique solution 
\begin{equation}
    V
    = \frac{\alpha^2\sigma_{j,k}^2}{1-a_{j,k}^2}
    = \frac{\alpha^2\sigma_{j,k}^2}{d_{j,k}(2-d_{j,k})},
\end{equation}
where the last equality uses $1-a_{j,k}^2=(1-a_{j,k})(1+a_{j,k})=d_{j,k}(2-d_{j,k})$.
\qed

\subsubsection{Derivation for Remark~\ref{remark:modal-orders}}

\textit{Derivation.}
Since $0<d_{j,k}\leq d_{\max}<2$, the bounds $2-d_{\max}\leq 2-d_{j,k}\leq 2$ applied
to Theorem~\ref{theorem:modal-variance} give
\begin{equation}
    \frac{\alpha^2\sigma_{j,k}^2}{2d_{j,k}}
    \leq
    \mathbb{E}[z_{j,k}^2]
    \leq
    \frac{\alpha^2\sigma_{j,k}^2}{(2-d_{\max})d_{j,k}}.
\end{equation}
For mixing-dominated modes, $|\alpha\lambda_k|\leq\rho\gamma\mu_j$ implies
\begin{equation}
    (1-\rho)\gamma\mu_j
    \leq d_{j,k}
    \leq
    (1+\rho)\gamma\mu_j.
\end{equation}
Together with $d_{j,k}\leq d_{\max}<2$, this makes $d_{j,k}(2-d_{j,k})=\Theta(\gamma\mu_j)$.
Substituting into Eq.~(\ref{eq:stationary-modal-variance}) proves the low-curvature claim.
For $\lambda_k>0$, the denominator is
\begin{equation}
    d_{j,k}(2-d_{j,k})
    =
    \Theta(d_{j,k})
    =
    \Theta(\gamma\mu_j+\alpha\lambda_k),
\end{equation}
where the constants depend on $d_{\max}$.
The same identity remains true for negative $\lambda_k$ whenever $0<d_{j,k}<2$, but negative curvature decreases $d_{j,k}$ relative to $\gamma\mu_j$.
\qed

\subsubsection{Proof of Corollary~\ref{corollary:radius-control}}

\textit{Proof.}
By Parseval's identity for the orthonormal bases $V$ and $U$,
\begin{equation}
    \|\Delta\|_F^2
    =
    \sum_{j=1}^{n-1}\sum_{k=1}^d z_{j,k}^2.
\end{equation}
Remark~\ref{remark:modal-orders} gives
\begin{equation}
    \mathbb{E}[z_{j,k}^2]
    =
    \Theta\left(\frac{\alpha^2\sigma_{j,k}^2}{\gamma\mu_j}\right)
\end{equation}
for every $(j,k)\in\mathcal{S}_{\mathrm{low}}$.
Summing over the active low-curvature modes yields
\begin{equation}
    \sum_{(j,k)\in\mathcal{S}_{\mathrm{low}}}\mathbb{E}[z_{j,k}^2]
    =
    \Theta\left(\frac{\alpha^2}{\gamma}
    \sum_{(j,k)\in\mathcal{S}_{\mathrm{low}}}\frac{\sigma_{j,k}^2}{\mu_j}\right)
    =
    \Theta\left(\frac{\alpha^2}{\gamma}M_{\mathrm{low}}(\alpha,\gamma)\right).
\end{equation}
If $M_{\mathrm{low}}(\alpha,\gamma)$ is bounded above and below by positive constants in the local window, this reduces to $\Theta(\alpha^2/\gamma)$.
For every $(j,k)\in\mathcal{S}_{\mathrm{high}}^+$, $\lambda_k\geq c>0$, so
\begin{equation}
    d_{j,k}
    =
    \gamma\mu_j+\alpha\lambda_k
    \geq
    \alpha c.
\end{equation}
Since $d_{j,k}\leq d_{\max}<2$, Eq.~(\ref{eq:stationary-modal-variance}) gives
\begin{equation}
    \mathbb{E}[z_{j,k}^2]
    \leq
    C\frac{\alpha^2\sigma_{j,k}^2}{\alpha c}
    =
    O(\alpha)
\end{equation}
for a constant $C$ independent of $\alpha$.
Summing over the finite set $\mathcal{S}_{\mathrm{high}}^+$ preserves the $O(\alpha)$ rate.
If $M_{\mathrm{low}}(\alpha,\gamma)$ remains non-negligible and $\alpha^2/\gamma=\Omega(\alpha)$, the low-curvature contribution decays no faster than this high-curvature upper bound.
\qed

\subsubsection{Proof of Corollary~\ref{corollary:gamma-schedule}}

\textit{Proof.}
If $\gamma=\alpha^p$ up to a constant factor, the mixing-dominated condition
\begin{equation}
    |\alpha\lambda_k|\leq\rho\gamma\mu_j
\end{equation}
is equivalent to $|\lambda_k|\leq C\rho\mu_j\alpha^{p-1}$ for a fixed constant $C>0$.
For such modes, Remark~\ref{remark:modal-orders} gives
\begin{equation}
    \mathbb{E}[z_{j,k}^2]
    =
    \Theta\left(\frac{\alpha^2\sigma_{j,k}^2}{\alpha^p\mu_j}\right)
    =
    \Theta\left(\alpha^{2-p}\sigma_{j,k}^2/\mu_j\right).
\end{equation}
For $p=2$ this is constant order in $\alpha$ for each active mixing-dominated mode; for $p>2$ it grows in the linear stationary approximation for each such mode.
The condition $|\lambda_k|\leq O(\alpha^{p-1})$ shrinks as $\alpha$ decreases, so the statement applies only to modes that remain active in this mixing-dominated regime.
\qed

\newpage
\subsection{Additional experiments and details}

\subsubsection{Decentralized Adam with adaptive consensus}\label{app:dadam-exp}

We also validate the idea of controlling consensus errors on transformer models by simply replacing the local update with the Adam optimizer \citep{kingma2014adam}. DSGD-AC is then adapted to DAdam-AC as in Algorithm~\ref{alg:dadam-ac}.

\begin{algorithm2e}
    \caption{Decentralized Adam with adaptive consensus (DAdam-AC) on worker $i$}
    \label{alg:dadam-ac}
    \SetAlgoLined
    \KwData{Dataset ($D$), the number of workers ($N$), the number of epoch ($E$), the number of batches per epoch ($T$), intialization ($x^{(0)}$), and a hyperparameter ($p\in\mathbb{R}^+$).}

    \KwResult{Deployed model $\bar{x}=\frac{1}{n}\sum_{j=1}^n x_j^{(TE)}$}

    $x_1^{(0)}=x_2^{(0)}=\cdots=x_n^{(0)}=x^{(0)}$

    \For{$t=1$ to $TE$} {
        $g_i^{(t)}=\nabla f(x_i^{(t-1)};s_i^{(t)})$

        $m_i^{(t)}=\beta_1 m_i^{(t)} + g_i^{(t)}$

        $v_i^{(t)}=\beta_2 v_i^{(t)} + g_i^{(t)}\odot g_i^{(t)}$

        $\hat{m}_i^{(t)}=m_i^{(t)}/(1-\beta_1^t)$

        $\hat{v}_i^{(t)}=v_i^{(t)}/(1-\beta_1^t)$

        $\gamma^{(t)}=\left[\alpha^{(t)}/\alpha_{\max}\right]^p$

        $x_i^{(t)}=x_i^{(t-1)}-\alpha^{(t)}\frac{\hat{m}_t}{\hat{v}_t}+\gamma^{(t)}\sum_{j\in\mathcal{N}(i)} W_{ij}(x_j^{(t-1)}-x_i^{(t-1)})$
    }
\end{algorithm2e}

We train Transformer (the big variant, $\sim$213M parameters) \citep{vaswani2017attention} on WMT14 (English-to-German) \citep{bojar2014findings} and present the curves of training losses and BLEU scores on the test set. The BLEU scores \citep{papineni2002bleu} (which is used to evaluate the translation quality in the original paper) and the losses on the test set and the training set are reported in Table~\ref{tab:algo_comparison_wmt}.

\begin{figure}[h]
    \centering
    \includegraphics[width=.4\linewidth]{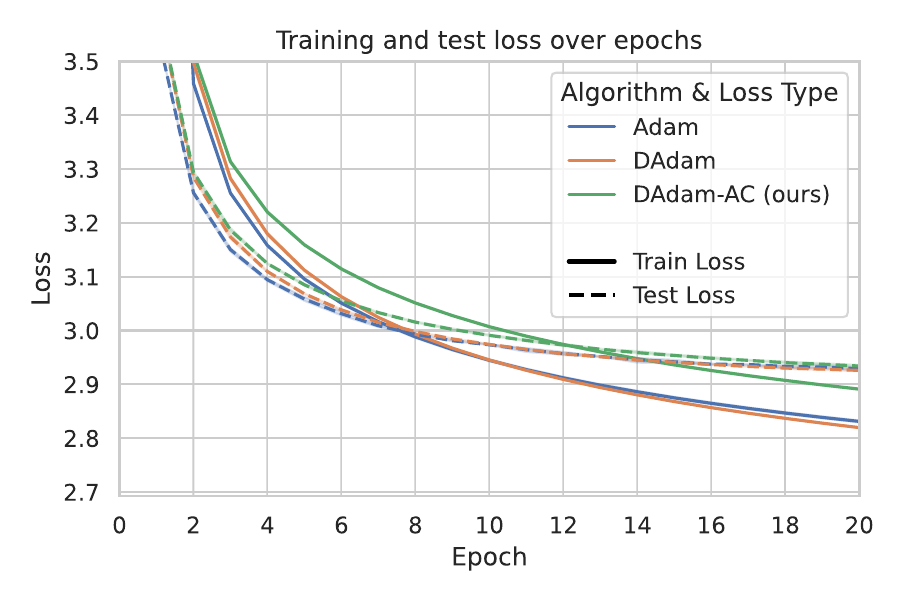}
    \includegraphics[width=.4\linewidth]{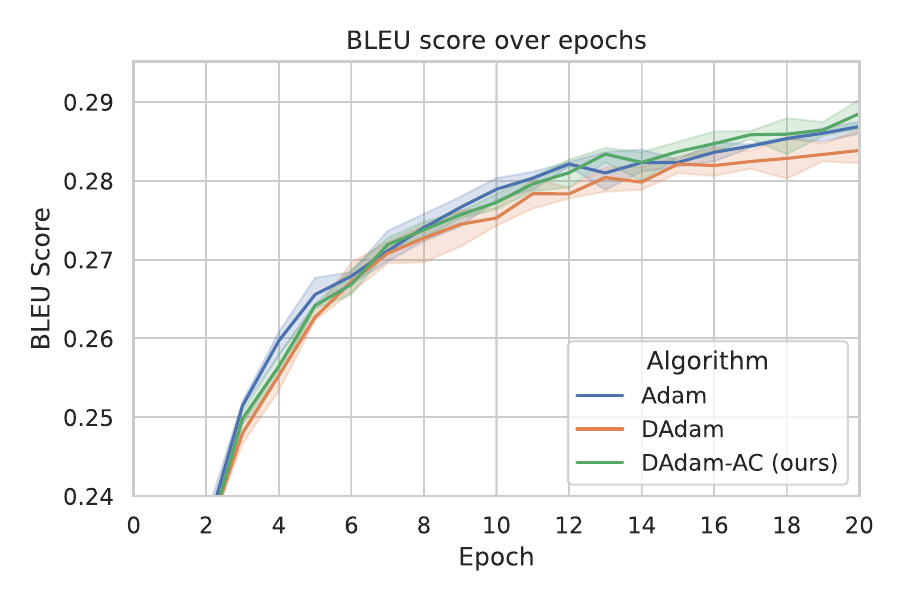}
    \caption{Transformer (big) on WMT14 English-to-German. \textbf{Left}: Losses on training set. \textbf{Right}: BLEU scores on the test set.}
    \label{fig:transformer-wmt14}
\end{figure}

\begin{table}[h]
\centering
{
\begin{tblr}{
  cells = {c},
  hline{1,5} = {-}{0.08em},
  hline{2} = {-}{0.05em},
}
\textbf{Algorithm} & \textbf{BLEU score~$\uparrow$}                        & \textbf{Test loss~$\downarrow$}                          & \textbf{Train loss~$\downarrow$}                         \\
Adam               & 28.68~\textcolor[rgb]{0,0.114,0.208}{±~}0.07          & 2.9290~\textcolor[rgb]{0,0.114,0.208}{±}~0.0026          & 2.8310~\textcolor[rgb]{0,0.114,0.208}{±}~0.0019          \\
DAdam              & 28.38~\textcolor[rgb]{0,0.114,0.208}{±}~0.22          & 2.9258~\textcolor[rgb]{0,0.114,0.208}{±}~0.0018 & \uline{2.8195~\textcolor[rgb]{0,0.114,0.208}{±}~0.0008} \\
DAdam-AC           & \uline{28.89~\textcolor[rgb]{0,0.114,0.208}{±}~0.17} & 	\uline{2.9205~\textcolor[rgb]{0,0.114,0.208}{±}~0.0020}          & 2.8456~\textcolor[rgb]{0,0.114,0.208}{±}~0.0016          
\end{tblr}}
\caption{Performance comparison of DAdam, Adam, and DAdam-AC on neural machine translation with the transformer model.}\label{tab:algo_comparison_wmt}
\end{table}

The results show that DAdam-AC can outperform other baselines on the translation quality metric and the test loss. The adaptive consensus enhances the decentralized Adam, but the improvement is relatively marginal compared with the image classification task. We believe further improvement is possible by taking adaptive consensus into account when designing the optimizer or the mixing operator (see the discussion in Section~\ref{sec:conclusion}).

\newpage 
\subsubsection{Training curves}\label{appsec:curves}

Figures~\ref{appfig:wrn28-10-cifar100-8w} to \ref{appfig:wrn16-8-cifar100-32w} present the training curves corresponding to the experiments in Table~\ref{tab:algo_comparison}. The figures show the impact of the adaptive consensus after the epoch it is activated, and the trend that DSGD-AC trades training loss for better test loss and test accuracy. Moreover, from the figures, one can observe that there is an increasing trend in test loss in the last 30-50 epochs of training. Based on the analysis in Section~\ref{sec:radius-control}, a heuristic explanation is that the regularization becomes weaker on the top eigenvalues due to the worse alignment at the late-training phase, which hinders the test performance.

\begin{figure}[H]
    \centering
    \includegraphics[width=\linewidth]{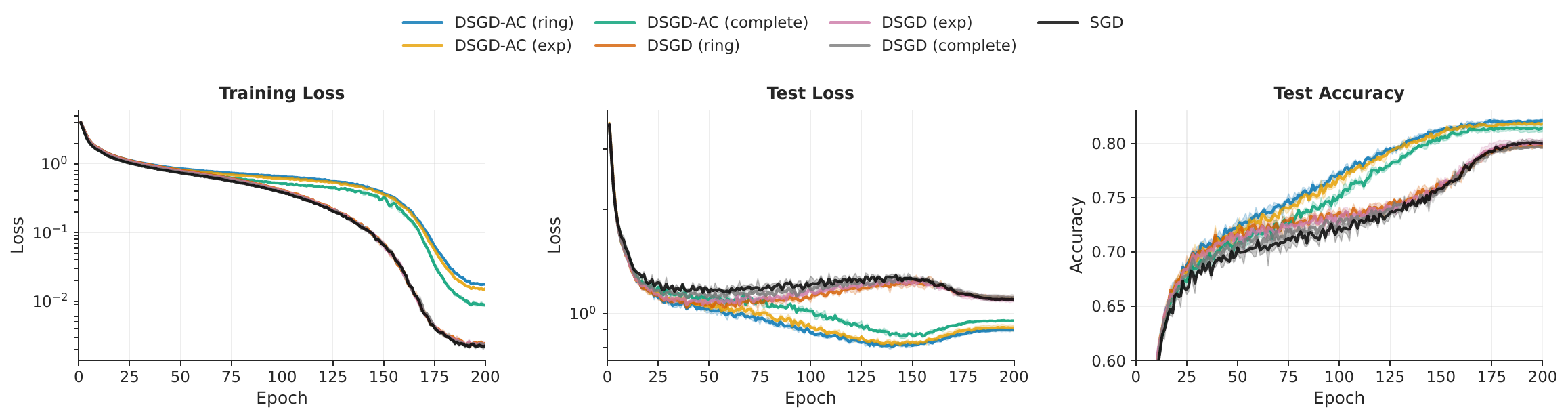}
    \caption{WRN28-10, CIFAR100, 8 workers}
    \label{appfig:wrn28-10-cifar100-8w}
\end{figure}

\begin{figure}[H]
    \centering
    \includegraphics[width=\linewidth]{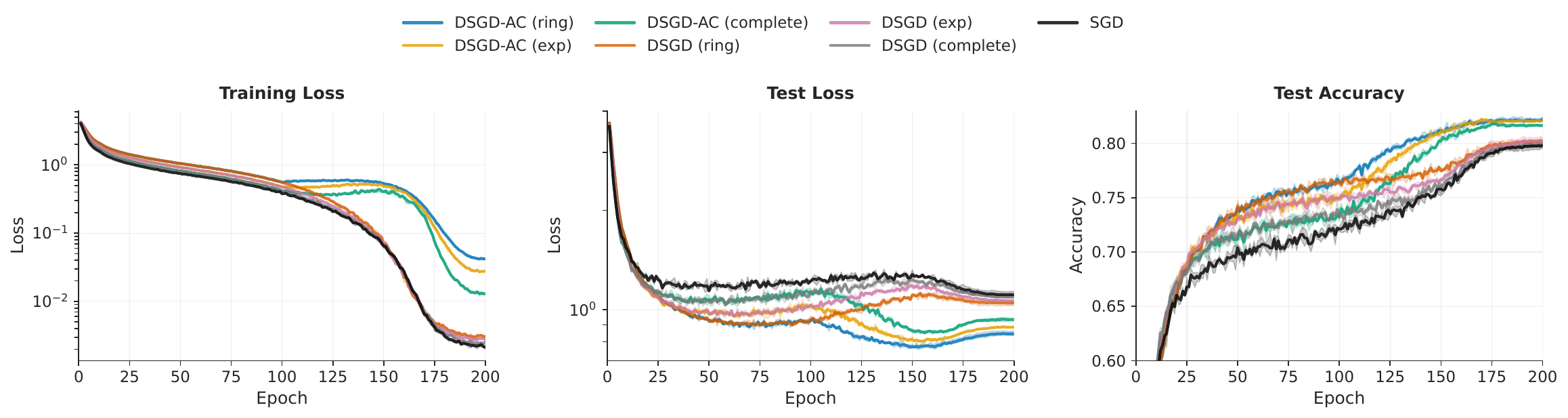}
    \caption{WRN28-10, CIFAR100, 16 workers}
\end{figure}

\begin{figure}[H]
    \centering
    \includegraphics[width=\linewidth]{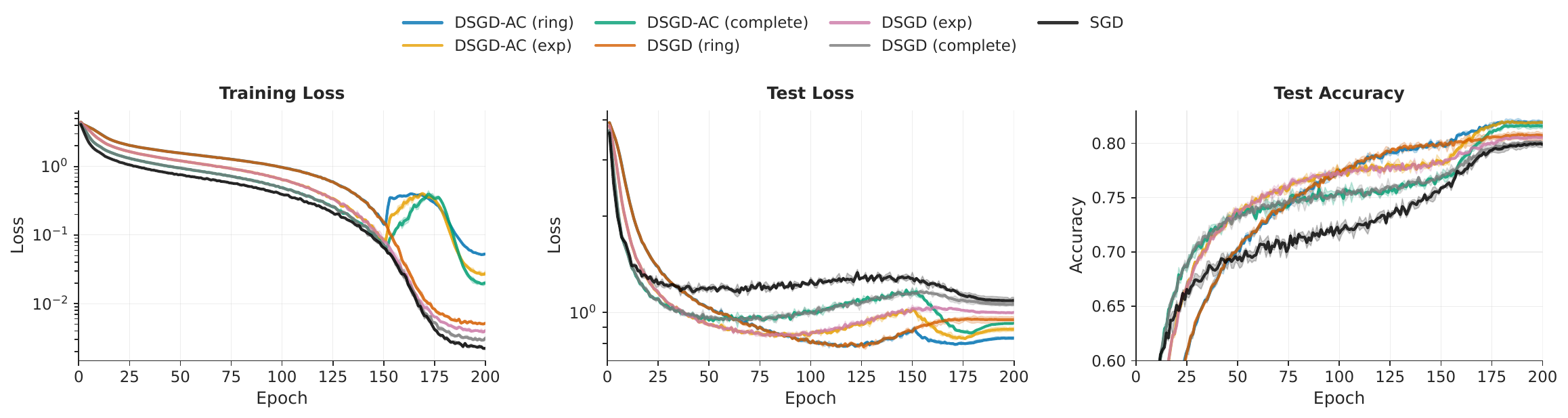}
    \caption{WRN28-10, CIFAR100, 32 workers}
\end{figure}
\newpage
\begin{figure}[H]
    \centering
    \includegraphics[width=\linewidth]{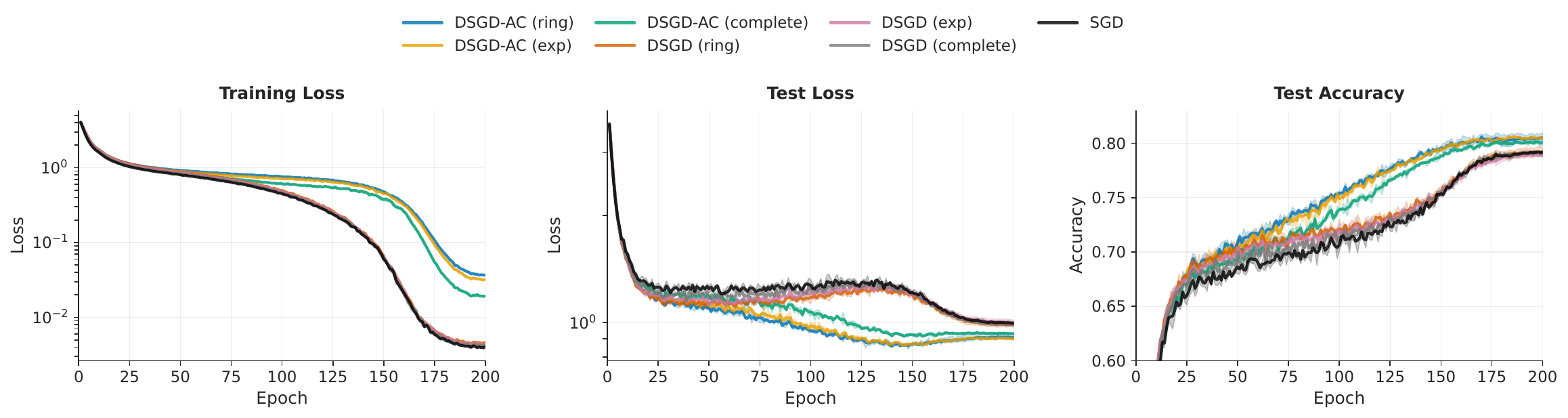}
    \caption{WRN16-8, CIFAR100, 8 workers}
\end{figure}

\begin{figure}[H]
    \centering
    \includegraphics[width=\linewidth]{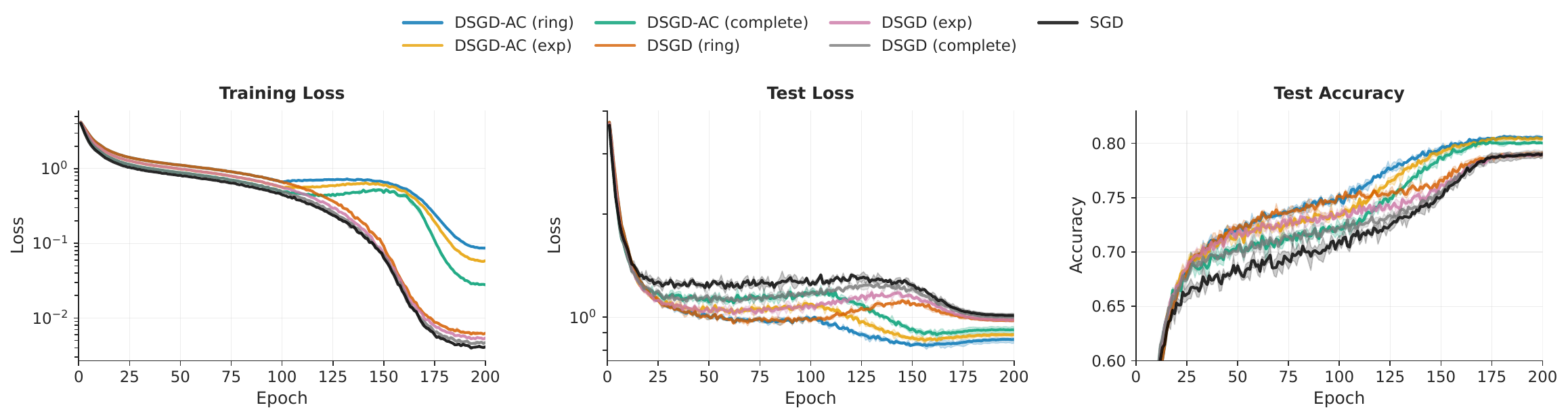}
    \caption{WRN16-8, CIFAR100, 16 workers}
\end{figure}

\begin{figure}[H]
    \centering
    \includegraphics[width=\linewidth]{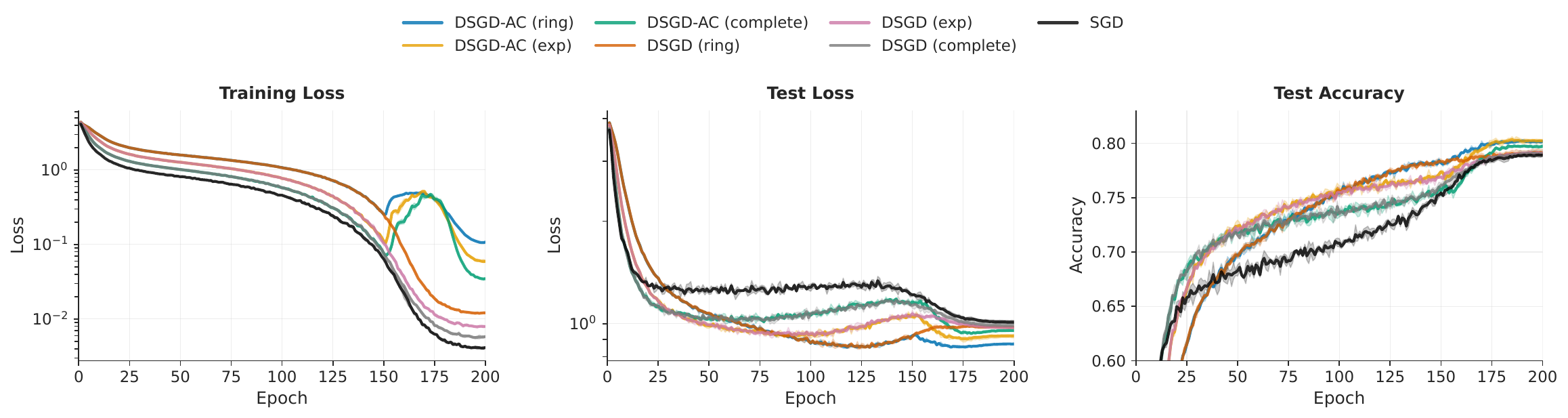}
    \caption{WRN16-8, CIFAR100, 32 workers}
    \label{appfig:wrn16-8-cifar100-32w}
\end{figure}

\newpage

\subsubsection{Sensitivity analysis}\label{appsec:sensitivity}

Figures~\ref{appfig:p-sensitivity} and \ref{appfig:start-sensitivity} show the curves of the test accuracy and the test loss of the sensitivity analysis experiments.

For $p$, a greater-than-zero $p$ gives better test performance, while the benefits diminish as $p>3$. 

$E_{\text{start}}$ is the epoch from which the adaptive consensus is activated. We set $\gamma=1$ before $E_{\text{start}}$, and make $\alpha_{\text{max}}$ equal to the maximal learning rate among the iterations where DSGD-AC is activated.

The results show that $E_{\text{start}}=100$ gives the best performance on test accuracy, and, for $E_{\text{start}}\in\{10,50,75,100\}$, they give similar test loss performance, and a too-late activation of the adaptive consensus ($E_\text{start}=\{150,175\}$) leaves too few iterations for DSGD-AC to take effect.
\begin{figure}[H]
    \centering
    \begin{subfigure}[t]{0.49\textwidth}
        \centering
        \includegraphics[width=\linewidth]{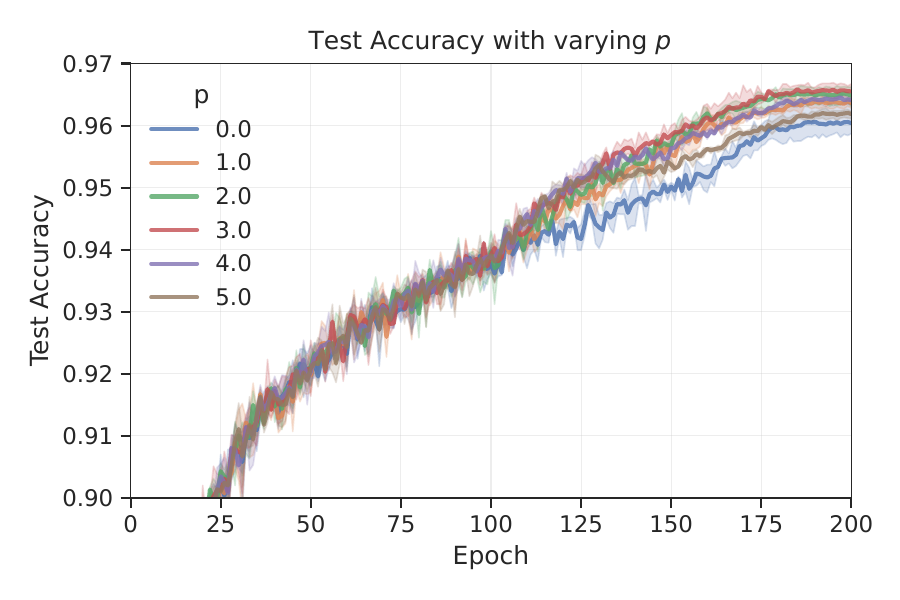}
    \end{subfigure}
    \hfill
    \begin{subfigure}[t]{0.49\linewidth}
        \centering
        \includegraphics[width=\linewidth]{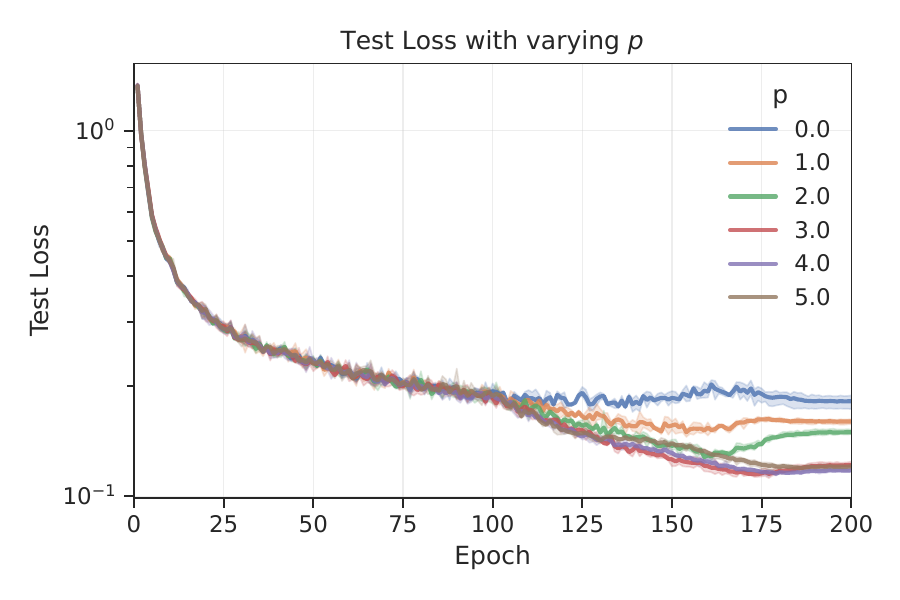}
    \end{subfigure}
    \caption{DSGD-AC with WRN28-10 on CIFAR10 with varying $p$ and $E_{\text{start}}=100$.}
    \label{appfig:p-sensitivity}
\end{figure}

\begin{figure}[H]
    \centering
    \begin{subfigure}[t]{0.49\textwidth}
        \centering
        \includegraphics[width=\linewidth]{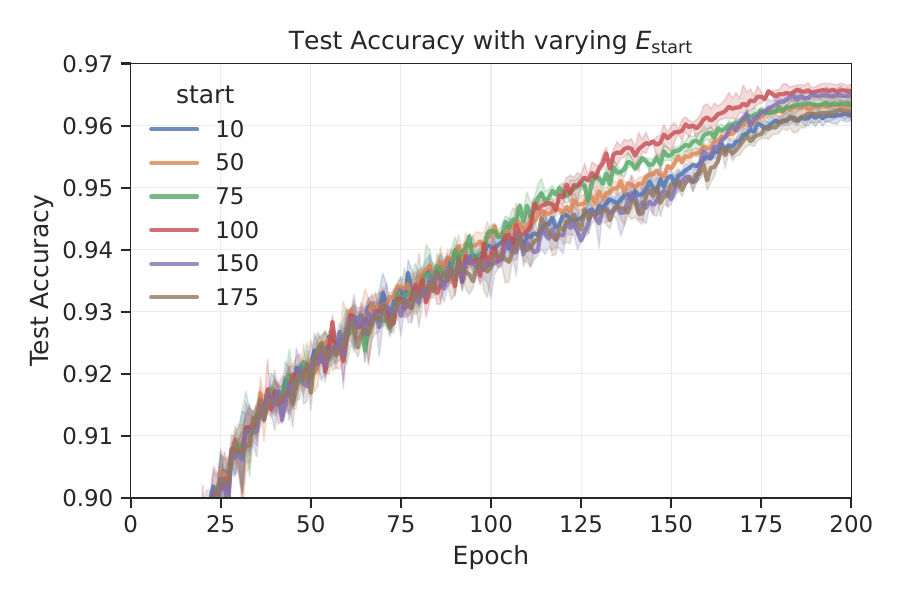}
    \end{subfigure}
    \hfill
    \begin{subfigure}[t]{0.49\linewidth}
        \centering
        \includegraphics[width=\linewidth]{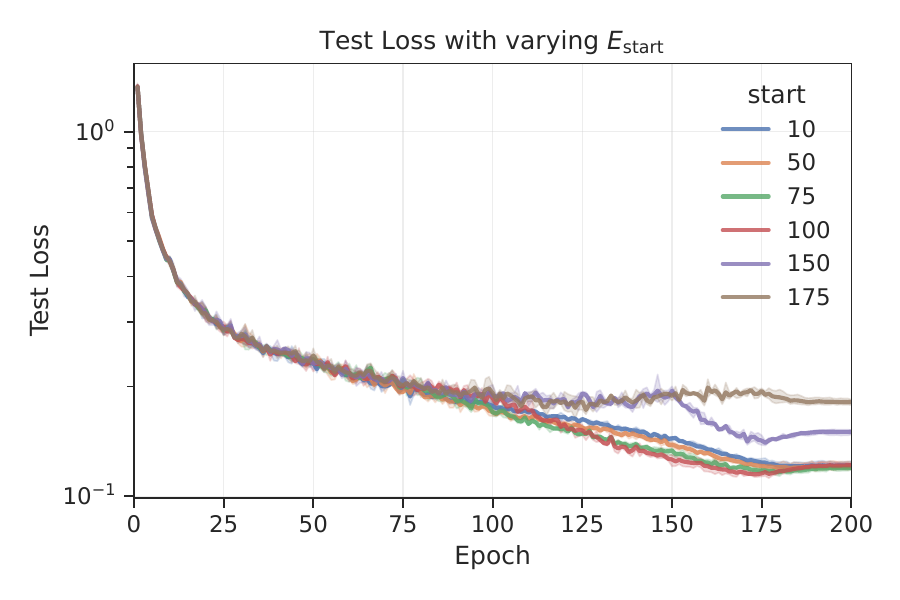}
    \end{subfigure}
    \caption{DSGD-AC with WRN28-10 on CIFAR10 with varying $E_{\text{start}}$ and $p=3$.}
    \label{appfig:start-sensitivity}
\end{figure}

\newpage
\subsubsection{Sharpness-aware minimization}\label{appsec:sam}

We implement SAM \citep{foret2020sharpness} and report its results with 8 and 16 workers (1 and 2 8$\times$T4 nodes) for reference. We follow their work to use $\rho=0.05$ in all the experiments. The results show that SAM achieves better metric values than DSGD-AC. However, it should be noted that SAM introduces $2\times$ computation cost, and it is not numerically stable under the mixed precision training in our setup (the loss becomes NaN in around 30 epochs with AMP in fp16). Therefore, it takes $\sim4\times$ training time as the synchronous SGD in the 8-worker and 16-worker setups under the same number of iterations. Even under a roughly the same FLOP budget (100 epochs), SAM still takes around $2\times$ training time, and the accuracies are comparable to DSGD-AC.

\begin{table}[H]
\centering
{
\begin{tblr}{
  cells = {c},
  hline{1,6} = {-}{0.06em},
  hline{2} = {-}{0.04em},
}
\textbf{\# Nodes}  &  \textbf{Epochs} & \textbf{Test Acc. (\%) $\uparrow$}  & \textbf{Test Loss~$\downarrow$}  & \textbf{Train Loss~$\downarrow$} & \textbf{Training Time (min) $\downarrow$} \\
1      & 200      &    83.04 {\tiny± 0.22} & 0.700 {\tiny± 0.007}  & 0.0519 {\tiny± 0.0003} & 382.14 {\tiny± 8.32}      \\
2      & 200      &    83.00 {\tiny± 0.03} & 0.697 {\tiny± 0.004}  & 0.0517 {\tiny± 0.0003}  & 193.97 {\tiny± 1.18} \\
1      & 100      &    82.49 {\tiny± 0.09} &  0.692 {\tiny± 0.002} & 0.0819 {\tiny± 0.0007} & 196.10 {\tiny± 1.33}      \\
2      & 100      &    82.48 {\tiny± 0.10} & 0.685 {\tiny± 0.005}  & 0.0836 {\tiny± 0.0007}  & 96.61 {\tiny± 0.78}
\end{tblr}}
\vspace{0.5em}
\caption{Performance of SAM on image classification tasks with WRN28-10 on CIFAR-100.}\label{tab:sam-results}
\end{table}

\newpage
\subsection{Hyperparameter details}\label{app:hyper-details}

\subsubsection{Hyperparameters for image classification experiments with WRN on CIFAR10/100}

The selection of hyperparameters follows the original paper \citep{zagoruyko2016wide}, and our baseline implementation is consistent with its performance.

\begin{table}[h]
\centering
\resizebox{0.85\linewidth}{!}{%
\begin{tabular}{p{4cm} p{10cm}}
\toprule
\textbf{Category} & \textbf{Setting} \\
\midrule
\multicolumn{2}{l}{\textbf{\textit{General}}} \\
Number of epochs & 200 \\
Global batch size & 128 for 8-worker setup and it is linearly scaled with the number of workers. \\
Learning rate scheduler & Linearly warm-up to $\alpha_{\max}$ in the first 10 epochs, followed by the cosine annealing until the end. $\alpha_{\max}=0.1$ for 128 batch size, and it is linearly scaled with the batch size. \\
Base optimizer & SGD with momentum $\beta=0.9$ and weight decay $5\times 10^{-4}$. \\
Data shuffle & Randomly shuffled and split into $N$ local datasets each epoch. \\
\midrule
\multicolumn{2}{l}{\textbf{\textit{Decentralized training}}} \\
Number of workers & 8, 16, and 32 \\
Communication topology & One-peer ring (alternating between neighbors $i-1$ and $i+1$ across iterations), one-peer exponential \cite{ying2021exponential}, and complete graph (global all-reduce among all workers). \\
DSGD-AC parameters & Exponent $p=3$; $\gamma=1$ before $E_\text{start}$. \\
BatchNorm calibration & Similar to the case in \cite{defazio2024road}, a calibration on the BatchNorm statistics is needed because there is a mismatch between the local models and the global average. To calibrate mismatched statistics, a full pass over the training set is conducted before validation. Only one calibration should be done if intermediate checkpoints are not evaluated. Note that we also apply the calibration to synchronous SGD for a fair comparison.\\
\bottomrule
\end{tabular}
}
\end{table}

\subsubsection{Hyperparameters for neural machine translation experiments with Transformer on WMT14}

The selection of hyperparameters follows the original paper \citep{vaswani2017attention}, and our baseline implementation matches its performance. We tuned the learning rate and $(\beta_1,\beta_2)$ separately for decentralized methods.

\begin{table}[h]
\centering
\resizebox{0.85\linewidth}{!}{%
\begin{tabular}{p{4cm} p{10cm}}
\toprule
\textbf{Category} & \textbf{Setting} \\
\midrule
\multicolumn{2}{l}{\textbf{\textit{General}}} \\
Number of epochs & 20 \\
Global batch size & $\sim$50k tokens including both source and target texts \\
Learning rate scheduler & Linear warm-up to $5\times 10^{-4}$ over the first 4000 iterations, then decay as $\alpha_0\cdot(4000/t)^{0.5}$ ($t$ is the iteration index). $\alpha_0=0.0005$ for centralized Adam, and $\alpha_0=0.0013$ for decentralized methods. \\
Base optimizer & Adam ($\beta_1=0.9, \beta_2=0.98$) for centralized Adam, and ($\beta_1=0.974$, $\beta_2=0.999$) for decentralized methods. \\
Data shuffle & Randomly shuffled and split into $N$ local datasets each epoch \\
\midrule
\multicolumn{2}{l}{\textbf{\textit{Decentralized training}}} \\
Number of workers & 8 \\
Communication topology & One-peer ring (alternating between neighbors $i-1$ and $i+1$ across iterations) \\
DSGD-AC parameters. & Exponent $p=2$ (tuned based on experiments); $\gamma=1$ during warm-up. \\
Normalization & Since only layer normalization is used, no calibration is needed. \\
\bottomrule
\end{tabular}
}
\end{table}

\subsubsection{Evaluation details}\label{appsec:eval-details}

\paragraph{Eigenvalue evaluation} We use the Lanczos algorithm \citep{lanczos1950iteration} for efficient evaluation of the max and min eigenvalues of the Hessian. To be specific, the loss function is
\begin{equation}\label{appeq:true-loss}
    F(x)=\frac{1}{B}\sum_{b=1}^B f(x;s_b)
\end{equation}
where $s_b$ denotes a batch of samples from the training set, $\mathcal{S}$ denotes the total number of samples in the training set, and $\{s_1,\cdots,s_B\}$ collectively cover the entire training set. To exclude the impact of mismatched statistics in the batch normalization layers and to reflect the actual training loss, we evaluate the Hessian-vector product in training mode and control the random seed to ensure that the sample order, batching, and data augmentation are identical across all evaluations. In our evaluation, we fix the mini-batch size to 32.

We fix the number of Lanczos iterations to 30.

\paragraph{Alignment evaluation} The Hessian quadratic with any vector can be computed via the dot product between the consensus error and the Hessian vector product, and the Hessian vector product (HVP) is computed using the Pearlmutter trick without materializing the Hessian, which is supported out-of-the-box by the automatic differentiation in PyTorch.

\begin{equation}
    D:= \nabla F(x)^\top v, \quad \frac{\partial D}{\partial x}= H(x)v
\end{equation}
where $F(x)$ is defined in the same way as in Eq. (\ref{appeq:true-loss}), and, since $F(x)$ is a linear combination of loss functions on batches, the HVP can be computed in a distributed way for better efficiency.

We use 50 Hutchinson trace probes to evaluate the alignment between the Hessian and random directions in our showcase runs, and the standard deviations are on the order of $10^{-7}$.

\paragraph{Gradient noise alignment evaluation} The gradient noises and their alignment with the Hessian are computed by
\begin{equation}
    \xi_i := \nabla f(x;s_i) - \frac{1}{B}\sum_{b=1}^B \nabla f(x;s_b),\quad A_{\textrm{grad}}:= \frac{1}{B}\sum_{b=1}^B \xi_b^\top H \xi_b / \lambda_1(H)
\end{equation}
where the first step is simply computing the difference between the mini-batch loss and the full-batch loss, and the second step is computed via HVP as in the alignment evaluation.

\end{document}